%% file: main.tex
\documentclass{article}

\usepackage[margin=1in]{geometry}

\usepackage[utf8]{inputenc} %
\usepackage[T1]{fontenc}    %
\usepackage{url}            %
\usepackage{booktabs}       %
\usepackage{amsfonts}       %
\usepackage{nicefrac}       %
\usepackage{microtype}      %
\usepackage{xcolor}         %

\title{
On the non-universality of deep learning:\\ quantifying the cost of symmetry
}

\author{%
  Emmanuel~Abbe \\ EPFL
   \and
   Enric Boix-Adser\`a \\ MIT
}

\input{importsanddeclares.tex}

\begin{document}

\maketitle

\begin{abstract}

We prove limitations on what neural networks trained by noisy gradient descent (GD) can efficiently learn. Our results apply whenever GD training is equivariant, which holds for many standard architectures and initializations. As applications, (i) we characterize the functions that fully-connected networks can weak-learn on the binary hypercube and unit sphere, demonstrating that depth-2 is as powerful as any other depth for this task; (ii) we extend the merged-staircase necessity result for learning with latent low-dimensional structure \cite{abbe2022merged} to beyond the mean-field regime. Under cryptographic assumptions, we also show hardness results for learning with fully-connected networks trained by stochastic gradient descent (SGD).

\end{abstract}

\section{Introduction}

Over the last decade, deep learning has made advances in areas as diverse as image classification \cite{krizhevsky2012imagenet}, language translation \cite{bahdanau2014neural}, classical board games \cite{silver2018general}, and programming \cite{li2022competition}. Neural networks trained with gradient-based optimizers have surpassed classical methods for these tasks, raising the question: can we hope for deep learning methods to eventually replace all other learning algorithms? In other words, is deep learning a universal learning paradigm? Recently, \cite{abbe2020poly,abbe2021power} proved that in a certain sense the answer is yes: any PAC-learning algorithm \cite{valiant1984theory} can be efficiently implemented as a neural network trained by stochastic gradient descent; analogously, any Statistical Query algorithm \cite{kearns1998efficient} can be efficiently implemented as a neural network trained by noisy gradient descent.

However, there is a catch: the result of \cite{abbe2020poly} relies on a carefully crafted network architecture with memory and computation modules, which is capable of emulating an arbitrary learning algorithm. This is far from the architectures which have been shown to be successful in practice. Neural networks in practice do incorporate domain knowledge, but they have more ``regularity'' than the architectures of \cite{abbe2020poly}, in the sense that they do not rely on heterogeneous and carefully assigned initial weights (e.g., convolutional networks and transformers for image recognition and language processing \cite{lecun1995convolutional,lecun2010convolutional,vaswani2017attention}, graph neural networks for analyzing graph data \cite{gori2005new,bruna2013spectral,velickovic2017graph}, and networks specialized for particle physics \cite{bogatskiy2020lorentz}). We therefore refine our question:

\begin{center}
    {\em Is deep learning with  ``regular'' architectures and initializations a universal learning paradigm? \\
    If not, can we quantify its limitations when architectures and data are not well aligned? }
\end{center}

We would like an answer applicable to a wide range of architectures. In order to formalize the problem and develop a general theory, we take an approach similar to \cite{ng2004feature,shamir2018distribution,li2021convolutional} of understanding deep learning through the \textit{equivariance group} $G$ (a.k.a., symmetry group) of the learning algorithm. 
\begin{definition}[$G$-equivariant algorithm]\label{def:equi-intro}
A randomized algorithm $\cA$ that takes in a data distribution $\cD \in \cP(\cX \times \cY)$\footnote{The set of probability distributions on $\Omega$ is denoted by $\cP(\Omega)$. You should think of $\cD \in \cP(\cX \times \cY)$ as a distribution of pairs $(\bx,y)$ of covariates and labels.} and outputs a function $\cA(\cD) : \cX \to \cY$ is said to be $G$-equivariant if for all $g \in G$
\begin{align}
\cA(\cD) \stackrel{d}{=} \cA(g(\cD)) \circ g. \tag{$G$-equivariance}
\end{align}
Here $g$ is a group element that acts on the data space $\cX$, and so is viewed as a function $g : \cX \to \cX$, and $g(\cD)$ is the distribution of $(g(\bx),y)$, where $(\bx,y) \sim \cD$.
\end{definition}
In the case that the algorithm $\cA$ is deep learning on the distribution $\cD$, the equivariance group depends on the optimizer, the architecture, and the network initialization \cite{ng2004feature,li2021convolutional}.\footnote{Note that the equivariance group of a \textit{training algorithm} should not be confused with the equivariance group of an \textit{architecture} in the context of geometric deep learning \cite{bronstein2021geometric}. In that context, $G$-equivariance refers to the property of a neural network architecture $\fNN(\cdot;\btheta) : \cX \to \cY$ that $\fNN(g(\bx);\btheta) = g(\fNN(\bx;\btheta))$ for all 
$\bx \in \cX$ and all group elements $g \in G$. In that case, $G$ acts on both the input in $\cX$ and output in $\cY$.}

\paragraph{Examples of $G$-equivariant algorithms in deep learning} In many deep learning settings, the equivariance group of the learning algorithm is large. Thus, in this paper, we call an algorithm ``regular'' if it has a large equivariance group. For example, SGD training of fully-connected networks with Gaussian initialization is orthogonally-equivariant \cite{ng2004feature}; and is  permutation-equivariant if we add skip connections \cite{he2016deep}.  SGD training of convolutional networks is translationally-equivariant if circular convolutions are used \cite{schubert2019circular}, and SGD training of i.i.d.-initialized transformers without positional embeddings is equivariant to permutations of tokens \cite{vaswani2017attention}. Furthermore, \cite[Theorem C.1]{li2021convolutional} provides general conditions under which a deep learning algorithm is equivariant. See also the preliminaries in Section~\ref{sec:prelim}.

\paragraph{Summary of this work} Based off of $G$-equivariance, we prove limitations on what ``regular'' neural networks trained by noisy gradient descent (GD) or stochastic gradient descent (SGD) can efficiently learn, implying a separation with the initializations and architectures considered in \cite{abbe2020poly}. For GD, we prove a master theorem that enables two novel applications: (a) characterizing which functions can be efficiently weak-learned by fully-connected (FC) networks on both the hypercube and the unit sphere; and (b) a necessity result for which functions on the hypercube with latent low-dimensional structure can be efficiently learned. See Sections~\ref{sec:contrib1-informal} and \ref{sec:contrib2-informal} for more details.

\subsection{Related work}\label{ssec:related}
Most prior work on computational lower bounds for deep learning has focused on proving limitations of kernel methods (a.k.a. linear methods). Starting with \cite{barron1993universal} and more recently with \cite{wei2019regularization,allen2019can,kamath2020approximate,
allen2020backward,hsudimension,hsu2021approximation,
abbe2022merged} it is known that there are problems on which kernel methods provably fail. These results apply to training neural networks in the Neural Tangent Kernel (NTK) regime \cite{jacot2018neural}, but do not apply to more general nonlinear training. Furthermore, for specific architectures such as FC architectures \cite{ghorbani2021linearized,misiakiewicz2022spectrum} and convolutional architectures \cite{misiakiewicz2021learning}, the kernel and random features models at initialization are well understood, yielding stronger lower bounds for training in the NTK regime.

For nonlinear training, which is the setting of this paper, considerably less is known. In the context of sample complexity, \cite{ng2004feature} introduced the study of the equivariance group of SGD, and constructed a distribution on $d$ dimensions with a $\Omega(d)$ versus $O(1)$ sample complexity separation for learning with an SGD-trained FC architecture versus an arbitrary algorithm. More recently, \cite{li2021convolutional} built on \cite{ng2004feature} to show a $O(1)$ versus $\Omega(d^2)$ sample-complexity separation between SGD-trained convolutional and FC architectures. In this paper, we also analyze the equivariance group of the training algorithm, but with the goal of proving superpolynomial computational lower bounds.

In the context of computational lower bounds, it is known that networks trained with noisy\footnote{Here the noise is used to control the gradients' precision as in \cite{abbe2020poly,abbe2021power}.} gradient descent (GD) fall under the Statistical Query (SQ) framework \cite{kearns1998efficient}, which allows showing computational limitations for GD training based on SQ lower bounds. This has been combined in \cite{abbe2020poly,shalev2017failures,malach2020computational,abbe2022initial} with the permutation symmetry of GD-training of i.i.d. FC networks to prove impossibility of efficiently learning high-degree parities and polynomials. In our work, we show that these arguments can be viewed in the broader context of more general group symmetries, yielding stronger lower bounds than previously known. For stochastic gradient descent (SGD) training, \cite{abbe2022merged} proves a computational limitation for training of two-layer mean-field networks, but their result applies only when SGD converges to the mean-field limit, and does not apply to more general architectures beyond two-layer networks. Finally, most related to our SGD hardness result is \cite{shamir2018distribution}, which shows limitations of SGD-trained FC networks under a cryptographic assumption. However, the argument of \cite{shamir2018distribution} relies on training being equivariant to linear transformations of the data, and therefore requires that data be whitened or preconditioned. Instead, our result for SGD does not require any preprocessing steps.

There is also recent work showing sample complexity benefits of invariant/equivariant neural network \textit{architectures} \cite{mei2021learning,elesedy2021provablya,elesedy2021provablyb,bietti2021sample,elesedy2022group}. In contrast, we study equivariant training \textit{algorithms}. These are distinct concepts: a deep learning algorithm can be $G$-equivariant, while the neural network architecture is neither $G$-invariant nor $G$-equivariant. For example, a FC network is not invariant to orthogonal transformations of the input. However, if we initialize it with Gaussian weights and train with SGD, then the learning algorithm is equivariant to orthogonal transformations of the input (see Proposition~\ref{prop:equi} below).

\subsection{Contribution 1: Lower bounds for noisy gradient descent (GD)}\label{sec:contrib1-informal}

Consider the supervised learning setup where we train a neural network $\fNN(\cdot; \btheta) : \cX \to \R$ parametrized by $\btheta \in \R^p$ to minimize the mean-squared error on a data distribution $\cD \in \cP(\cX \times \R)$,
\begin{align}\label{eq:loss}
\ell_{\cD}(\btheta) = \E_{(\bx,y) \sim \cD}[(y - \fNN(\bx;\btheta))^2].
\end{align}

The noisy Gradient Descent (GD) training algorithm randomly initializes $\btheta^0 \sim \mu_{\btheta}$ for some initialization distribution $\mu_{\btheta} \in \cP(\R^p)$, and then iteratively updates the parameters with step size $\eta > 0$ in a direction $\bg_{\cD}(\btheta^k)$ approximating the population loss gradient, plus Gaussian noise $\bxi^k \sim \cN(0,\tau^2 \bI)$,
\begin{align}
\label{GD}
\btheta^{k+1} = \btheta^k - \eta \bg_{\cD}(\btheta^k) + \bxi^k. \tag{GD}
\end{align}
Up to a constant factor, $\bg_{\cD}(\btheta)$ is the population loss gradient, except we have clipped the gradients of the network with the projection operator $\Pi_{B(0,R)}$ to lie in the ball $B(0,R) = \{\bz : \|\bz\|_2 \leq R\} \subset \R^p$,\footnote{Note that if $\fNN$ is an $R$-Lipschitz model, then $\bg_{\cD}(\btheta)$ will simply be the population gradient of the loss.}
\begin{align*}
 \bg_{\cD}(\btheta) = -\E_{(\bx,y) \sim \cD}[(y-\fNN(\bx;\btheta)) (\Pi_{B(0,R)} \nabla_{\btheta}\fNN(\bx;\btheta))].
 \end{align*}
Clipping the gradients is often used in practice to avoid instability from exploding gradients (see, e.g., \cite{zhang2019gradient} and references within). In our context, clipping ensures that the injected noise $\bxi^k$ is on the same scale as the gradient $\nabla_{\btheta} \fNN$ of the network and so it controls the gradients' precision. Similarly to the works \cite{abbe2020poly,abbe2021power,abbe2022initial}, we consider noisy gradient descent training to be efficient if the following conditions are met.
\begin{definition}[Efficiency of GD, informal]
GD training is \textit{efficient} if the clipping radius $R$, step size $\eta$, and inverse noise magnitude $1/\tau$ are all polynomially-bounded in $d$, since then \eqref{GD} can be efficiently implemented using noisy minibatch SGD\footnote{Efficient implementability by minibatch SGD assumes bounded residual errors.}.
\end{definition}

We prove that some data distributions cannot be efficiently learned by $G$-equivariant GD training. For this, we introduce the $G$-alignment:
\begin{definition}[$G$-alignment]\label{def:g-alignment}
Let $G$ be a compact group, let $\mu_{\cX} \in \cP(\cX)$ be a distribution over data points, and let $f \in \Capitalltwo(\mu_{\cX})$ be a labeling function. The $G$-alignment of $(\mu_{\cX}, f)$ is:
\begin{align*}
\mathcal{C}((\mu_{\cX}, f); G) = \sup_{h} \E_{g \sim \mu_G}[\E_{\bx \sim \mu_{\cX}}[f(g(\bx))h(\bx)]^2],
\end{align*}
where $\mu_G$ is the Haar measure of $G$ and the supremum is over $h \in \Capitalltwo(\mu_{\cX})$ such that $\|h\|^2 = 1$.
\end{definition}

In our applications, we use tools from representation theory (see e.g., \cite{knapp1996lie}) to evaluate the $G$-alignment. Using the $G$-alignment, we can prove a master theorem for lower bounds:

\begin{theorem}[GD lower bound, informal statement of Theorem~\ref{thm:gd-lower}]\label{thm:gd-lower-informal}
Let $\cD_f \in \cP(\cX \times \R)$ be the distribution of $(\bx,f(\bx))$ for $\bx \sim \mu_{\cX}$. If $\mu_{\cX}$ is $G$-invariant\footnote{Meaning that if $\bx \sim \mu_{\cX}$, then for any $g \in G$, we also have $g(\bx) \sim \mu_{\cX}$.} and the $G$-alignment of $(\mu_{\cX}, f)$ is small, then $f$ cannot be efficiently learned by a $G$-equivariant GD algorithm.
\end{theorem}

\paragraph{Proof ideas} We first make an observation of \cite{ng2004feature}: if a $G$-equivariant algorithm can learn the function $f$ by training on the distribution $\cD_f$, then, for any group element $g \in G$, it can learn $f \circ g$ by training on the distribution $\cD_{f \circ g}$. In other words, the algorithm can learn the class of functions $\cF = \{f \circ g : g \in g\}$, which can potentially be much larger than just the singleton set $\{f\}$. We conclude by showing that the class of functions $\cF$ cannot be efficiently learned by GD training. The intuition is that the $G$-alignment measures the diversity of the functions in $\cF$. If the $G$-alignment is small, then there is no function $h$ that correlates with most of the functions in $\cF$, which can be used to show $\cF$ is hard to learn by gradient descent.

This type of argument appears in \cite{abbe2020poly,abbe2022initial} in the specific case of Boolean functions and for permutation equivariance; our proof both applies to a more general setting (beyond Boolean functions and permutations) and yields sharper bounds; see Appendix~\ref{app:cross-pred}. Our bound can also be interpreted in terms of the Statistical Query framework, as we discuss in Appendix~\ref{app:sq}. While Theorem~\ref{thm:gd-lower-informal} is intuitively simple, we demonstrate its power and ease-of-use by deriving two new applications.

\paragraph{Application: Characterization of weak-learnability by fully-connected (FC) networks}

In our first application, we consider weak-learnability: when can a function be learned non-negligibly better than just outputting the estimate $\fNN \equiv 0$? Using Theorem~\ref{thm:gd-lower-informal}, we characterize which functions over the binary hypercube $f : \{+1,-1\}^d \to \R$ and over the sphere $f : \S^{d-1} \to \R$ are efficiently weak-learnable by GD-trained FC networks with i.i.d. symmetric and i.i.d. Gaussian initialization, respectively. The takeaway is that a function $f : \{+1,-1\}^d \to \R$ is weak-learnable if and only if it has a nonnegligible Fourier coefficient of order $O(1)$ or $d - O(1)$. Similarly, a function $f : \S^{d-1} \to \R$ is weak-learnable if and only if it has nonnegligible projection onto the degree-$O(1)$ spherical harmonics. Perhaps surprisingly, such functions can be efficiently weak-learned by 2-layer fully-connected networks, which shows that adding more depth does not help. This application is presented in Section~\ref{sec:characterize}.

\paragraph{Application: Evidence for the staircase property}
In our second application, we consider learning a target function $f : \{+1,-1\}^d \to \R$ that only depends on the first $P$ coordinates, $f(\bx) = h(x_1,\ldots,x_P)$. Our regime of interest here is when the function $h and  : \{+1,-1\}^P \to \R$ remains fixed and the dimension $d$ grows, since this models the situation where a latent low-dimensional space determines the labels in a high-dimensional dataset. Recently, \cite{abbe2022merged} studied SGD-training of mean-field two-layer networks, and gave a near-characterization of which functions can be learned to arbitrary accuracy $\eps$ in $O_{h,\eps}(d)$ samples, in terms of the \textit{merged-staircase property} (MSP). Using Theorem~\ref{thm:gd-lower-informal}, we prove that the MSP is necessary for GD-learnability whenever training is permutation-equivariant (which applies beyond the 2-layer mean-field regime) and we also generalize it beyond leaps of size 1. Details are in Section~\ref{sec:msp}.

\subsection{Contribution 2: Hardness for stochastic gradient descent (SGD)}\label{sec:contrib2-informal}

The second part of this paper concerns Stochastic Gradient Descent (SGD) training, which randomly initializes the weights $\btheta^0 \sim \mu_{\btheta}$ , and then iteratively trains the parameters with the following update rule to try to minimize the loss \eqref{eq:loss}:
\begin{align}
\label{SGD}
\btheta^{k+1} = \btheta^k - \eta \nabla_{\btheta} (y - \fNN(\bx_{k+1}; \btheta))^2 \mid_{\btheta = \btheta^k} \tag{SGD},
\end{align}
where $(y_{k+1},\bx_{k+1}) \sim \cD$ is a fresh sample on each iteration, and $\eta > 0$ is the learning rate.\footnote{For brevity, we focus on one-pass SGD with a single fresh sample per iteration. Our results extend to empirical risk minimization (ERM) setting and to mini-batch SGD, see Remark~\ref{rem:beyond-sgd}.} 

Proving computational lower bounds for SGD is a notoriously difficult problem \cite{abbe2021power}, exacerbated by the fact that for general architectures SGD can be used to simulate any polynomial-time learning algorithm \cite{abbe2020poly}. However, we demonstrate that one can prove hardness results for SGD training based off of cryptographic assumptions when the training algorithm has a large equivariance group. We demonstrate the non-universality of SGD on a standard FC architecture. 

\begin{theorem}[Hardness for SGD, informal statement of Theorem~\ref{thm:sgd-sum-mod-8-hard}]
Under the assumption that the Learning Parities with Noise (LPN) problem\footnote{See Section~\ref{sec:results-sgd} and Appendix~\ref{app:lpgn} for definitions and discussion on LPN.} is hard, FC neural networks with Gaussian initialization trained by SGD cannot learn 
$\fmodeight : \{+1,-1\}^d \to \{0,\ldots,7\}$, $$\fmodeight(\bx) \equiv \sum_{i=1}^d x_i \pmod{8},$$ in polynomial time from noisy samples $(\bx, \fmodeight(\bx) + \xi)$ where $\bx \sim \{+1,-1\}^d$ and $\xi \sim \cN(0,1)$.
\end{theorem}
This result shows a limitation of SGD training based on an average-case reduction from a cryptographic problem. The closest prior result is in \cite{shamir2018distribution}, which proved hardness results for learning with SGD on FC networks, but required preprocessing the data with a whitening transformation.

\paragraph{Proof idea} The FC architecture and Gaussian initialization are necessary: an architecture that outputted $\fmodeight(\bx)$ at initialization would trivially achieve zero loss. However, SGD on Gaussian-initialized FC networks is \textit{sign-flip} equivariant, and this symmetry makes $\fmodeight$ hard to learn. If a sign-flip equivariant algorithm can learn the function $\fmodeight(\bx)$ from noisy samples, then it can learn the function $\fmodeight(\bx \odot \bs)$ from noisy samples, where $\bs \in \{+1,-1\}^d$ is an unknown sign-flip vector, and $\odot$ denotes elementwise product. However, this latter problem is hard under standard cryptographic assumptions. More details in Section~\ref{sec:results-sgd}.

\section{Preliminaries}\label{sec:prelim}

\paragraph{Notation}
Let $\cH_d = \{+1,-1\}^d$ be the binary hypercube, and $\S^{d-1} = \{\bx \in \R^d : \|\bx\|_2 = 1\}$ be the unit sphere.
The law of a random variable $X$ is $\cL(X)$. If $S$ is a finite set, then $X \sim S$ stands for $X \sim \mathrm{Unif}[S]$. Also let $\bx \sim \S^{d-1}$ denote $\bx$ drawn from the uniform Haar measure on $\S^{d-1}$. For a set $\Omega$, let $\cP(\Omega)$ be the set of distributions on $\Omega$. Let $\odot$ be the elementwise product. For any $\mu_{\cX} \in \cP(\cX)$, and group $G$ acting on $\cX$, we say $\mu_{\cX}$ is $G$-invariant if $g(\bx) \stackrel{d}{=} \bx$ for $\bx \sim \mu_{\cX}$ and any $g \in G$.

\subsection{Equivariance of GD and SGD}

We define GD and SGD equivariance separately.

\begin{definition}\label{def:equi-gd}
Let $\AGD$ be the algorithm that takes in data distribution $\cD \in \cP(\cX \times \R)$, runs \eqref{GD} on initialization  $\btheta^0 \sim \mu_{\btheta}$ for $k$ steps, and outputs the function $\AGD(\cD) = \fNN(\cdot ;\btheta^k)$

We say ``$(\fNN, \mu_{\btheta})$-GD is $G$-equivariant'' if $\AGD$ is $G$-equivariant in the sense of Definition~\ref{def:equi-intro}.
\end{definition}

\begin{definition}\label{def:equi-sgd}
Let $\ASGD$ be the algorithm that takes in samples $(\bx_i,y_i)_{i \in [n]}$, runs \eqref{SGD} on initialization $\btheta^0 \sim \mu_{\btheta}$ for $n$ steps, and outputs $\ASGD((\bx_i,y_i)_{i \in [n]}) = \fNN(\cdot; \btheta^k)$.

We say ``$(\fNN, \mu_{\btheta})$-SGD is $G$-equivariant'' if $\ASGD((\bx_i,y_i)_{i \in [n]}) \stackrel{d}{=} \ASGD((g(\bx_i),y_i)_{i \in [n]}) \circ g$ for any $g \in G$, and any samples $(\bx_i,y_i)_{i \in [n]}$.
\end{definition}

\subsection{Regularity conditions on networks imply equivariances of GD and SGD} 
We take a data space $\cX \subseteq \R^d$, and consider the following groups that act on $\R^d$.
\begin{definition} Define the following groups and  actions:
\begin{itemize}
    \item Let $G_{perm} = S_d$ denote the group of permutations on $[d]$. An element $\sigma \in G_{perm}$ acts on $\bx \in \R^d$ in the standard way: $\sigma(\bx) = (x_{\sigma(1)},\ldots,x_{\sigma(d)})$.
    \item Let $\Gsignperm$ denote the group of signed permutations, an element $g = (\bs,\sigma) \in \Gsignperm$ is given by a sign-flip vector $\bs \in \cH_d$ and a permutation $\sigma \in G_{perm}$. It acts on $\bx \in \R^d$ by $g(\bx) = \bs \odot \sigma(\bx) = (s_1x_{\sigma(1)},\ldots,s_dx_{\sigma(d)})$.\footnote{The group product is $g_1g_2 = (\bs_1,\sigma_1)(\bs_2,\sigma_2) = (\bs_1 \odot \sigma_1(\bs_2), \sigma_1 \circ \sigma_2)$.}
    \item Let $\Grot = SO(d) \subseteq GL(d,\R)$ denote the rotation group. An element $g \in \Grot$ is a rotation matrix that acts on $\bx \in \R^d$ by matrix multiplication.
\end{itemize}
\end{definition}

Under mild conditions on the neural network architecture and initialization, GD and SGD training are known to be $G_{perm}$-, $\Gsignperm$-, or $\Grot$-equivariant~\cite{ng2004feature,li2021convolutional}.
\begin{assumption}[Fully-connected i.i.d. first layer and no skip connections from the input]\label{ass:noskip}
We can decompose the parameters as $\btheta = (\bW, \bpsi)$, where $\bW \in \R^{m \times d}$ is the matrix of the first-layer weights, and there is a function $\gNN(\cdot ; \bpsi) : \R^m \to \R$ such that
$\fNN(\bx;\btheta) = \gNN(\bW \bx; \bpsi)$. 
Furthermore, the initialization distribution is $\mu_{\btheta} = \mu_{\bW} \times \mu_{\bpsi}$, where $\mu_{\bW} = \mu_w^{\otimes (m \times d)}$ for $\mu_w \in \cP(\R)$.
\end{assumption}

Notice that Assumption~\ref{ass:noskip} is satisfied by FC networks with i.i.d. initialization. Under assumptions on $\mu_w$, we obtain equivariances of GD and SGD (see Appendix~\ref{app:equi} for proofs.)
\begin{proposition}[\cite{ng2004feature,li2021convolutional}]\label{prop:equi}
Under Assumption~\ref{ass:noskip}, GD and SGD are $G_{perm}$-equivariant. If $\mu_w$ is sign-flip symmetric, then GD and SGD are $\Gsignperm$-equivariant. If $\mu_w = \cN(0,\sigma^2)$ for some $\sigma$, then GD and SGD are $\Grot$-equivariant.
\end{proposition}

\section{Lower bounds for learning with GD}\label{sec:results-gd}

In this section, let $\cD(f,\mu_{\cX}) \in \cP(\cX \times \R)$ denote the distribution of $(\bx, f(\bx))$ where $\bx \sim \mu_{\cX}$.

We give a master theorem for computational lower bounds for learning with $G$-equivariant GD.

\begin{theorem}[GD lower bound using $G$-alignment]\label{thm:gd-lower}
Let $G$ be a compact group, and let $\fNN(\cdot; \btheta) : \cX \to \R$ be an architecture and $\mu_{\btheta} \in \cP(\R^p)$ be an initialization such that GD is $G$-equivariant.

Fix any $G$-invariant distribution $\mu_{\cX} \in \cP(\cX)$, any label function $f_* \in \Capitalltwo(\mu_{\cX})$, and any baseline function $\alpha \in \Capitalltwo(\mu_{\cX})$ satisfying $\alpha \circ g = \alpha$ for all $g \in G$. Let $\btheta^k$ be the random weights after $k$ time-steps of GD training with noise parameter $\tau > 0$, step size $\eta > 0$, and clipping radius $R > 0$ on the distribution $\cD = \cD(f_*, \mu_{\cX})$. Then, for any $\eps > 0$,
\begin{align*}
\PP_{\btheta^k} [\ell_{\cD}(\btheta^k) \leq \|f_* - \alpha\|^2_{\Capitalltwo(\mu_{\cX})} - \eps] \leq \frac{\eta R \sqrt{k \mathcal{C}}}{2 \tau} + \frac{ \mathcal{C}}{\eps},
\end{align*}
where $\mathcal{C} = \mathcal{C}((f_* - \alpha, \mu_{\cX}); G)$ is the $G$-alignment of Definition~\ref{def:g-alignment}.
\end{theorem}

As discussed in Section~\ref{sec:contrib1-informal}, the theorem states that if the $G$-alignment $\mathcal{C}$ is very small, then GD training cannot efficiently improve on the trivial loss from outputting $\alpha$: either the number of steps $k$, the gradient precision $R / \tau$, or the step size $\eta$ have to be very large in order to learn.  Appendix~\ref{sec:gd} shows a generalization of the theorem for learning a class of functions $\mathcal{\cF} = \{f_1,\ldots,f_m\}$ instead of just a single function $f_*$. This result goes beyond the lower bound of \cite{abbe2020poly} even when $G$ is the trivial group with one element: the main improvement is that Theorem~\ref{thm:gd-lower} proves hardness for learning real-valued functions beyond just Boolean-valued functions. We demonstrate the usefulness of the theorem through two new applications in Sections~\ref{sec:characterize} and \ref{sec:msp}. 

\subsection{Application: Characterizing weak-learnability by FC networks}\label{sec:characterize}

In our first application of Theorem~\ref{thm:gd-lower}, we consider FC architectures with i.i.d. initialization, and show how to use their training equivariances to characterize what functions they can weak-learn: i.e., for what target functions $f_*$ they can efficiently achieve a non-negligible correlation after training.

\begin{definition}[Weak learnability]\label{def:weak-learn}
Let $\{\mu_d\}_{d \in \NN}$ be a family of distributions $\mu_d \in \cP(\cX_d)$, and let $\{f_d\}_{d \in \NN}$ be a family of functions $f_d \in \Capitalltwo(\mu_d)$. Finally, let $\{\tilde{f}_d\}_{d \in \NN}$ be a family of estimators, where $\tilde{f}_d$ is a random function in $\Capitalltwo(\mu_d)$. We say that $\{f_d,\mu_d\}_{d \in \NN}$ is ``weak-learned'' by the family of estimators $\{\tilde{f}_d\}_{d \in \NN}$ if there are constants $d_0,C > 0$ such that for all $d > d_0$,
\begin{align}\label{eq:wl-condition}
\PP_{\tilde{f}_d}[\|f_d - \tilde{f}_d\|^2_{\Capitalltwo(\mu_d)} \leq \|f_d\|_{\Capitalltwo(\mu_d)}^2 - d^{-C}] \geq 9/10.
\end{align}
\end{definition}
The constant $9/10$ in the definition is arbitrary. In words, weak-learning measures whether the family of estimators $\{\tilde{f}_d\}$ has a non-negligible edge over simply estimating with the identically zero functions $\tilde{f}_d \equiv 0$. We study weak-learnability by GD-trained FC networks.
\begin{definition}\label{def:fully-conn-weak-learn}
We say that $\{f_d,\mu_d\}_{d \in \NN}$ is efficiently weak-learnable by GD-trained FC networks if there are FC networks and initializations $\{\fNNd,\mu_{\btheta,d}\}$, and hyperparameters $\{\eta_d,k_d,R_d,\tau_d\}$ such that for some constant $c > 0$,
\begin{itemize}
    \item Hyperparameters are polynomial size: $0 \leq \eta_d, k_d, R_d, 1/\tau_d \leq O(d^c)$; 
    \item  $\{\tilde{f}_d\}$ weak-learns $\{f_d,\mu_d\}$ in the sense of Definition~\ref{def:weak-learn}, where $\tilde{f}_d = \fNN(\cdot; \btheta_d)$ for weights $\btheta_d$ that are GD-trained on $\cD(f_d,\mu_d)$ for $k_d$ steps with step size $\eta_d$, clipping radius $R_d$, and noise $\tau_d$, starting from initialization  $\mu_{\btheta,d}$.
\end{itemize} 
If $\mu_{\btheta,d}$ is i.i.d copies of a symmetric distribution, we say that the FC networks are symmetrically-initialized, and Gaussian-initialized if $\mu_{\btheta,d}$ is i.i.d. copies of a Gaussian distribution.
\end{definition}

\subsubsection{Functions on hypercube, FC networks with i.i.d. symmetric initialization}

Let us first consider functions on the Boolean hypercube $f : \cH_d \to \R$. These can be uniquely written as a multilinear polynomial $$f(\bx) = \sum_{S \subseteq [d]} \hat{f}(S) \prod_{i \in S} x_i,$$ where $\hat{f}(S)$ are the Fourier coefficients of $f$ \cite{o2014analysis}. We characterize weak learnability of functions on the hypercube in terms of their Fourier coefficients. The full proof is deferred to Appendix~\ref{app:bool-char}.

\begin{theorem}\label{thm:bool-weak-char}
Let $\{f_d\}_{d \in \NN}$ be a family of functions $f_d : \cH_d \to \R$ with $\|f_d\|_{\Capitalltwo(\cH_d)} \leq 1$. Then $\{f_d,\cH_d\}$ is efficiently weak-learnable by GD-trained symmetrically-initialized FC networks if and only if there is a constant $C > 0$ such that for each $d \in \NN$ there is $S_d \subseteq [d]$ with $|S_d| \leq C$ or $|S_d| \geq d - C$, and $|\hat{f}_d(S_d)| \geq \Omega(d^{-C})$.
\end{theorem}
The algorithmic result can be achieved by two-layer FC networks, and relies on random features analysis where each network weight is initialized to $0$ with probability $1-p$, and $+1$ or $-1$ with equal probability $p/2$.\footnote{Surprisingly, this means that the full parity function $f_*(\bx) = \prod_{i=1}^d x_i$ can be efficiently learned with such initializations. See Appendix~\ref{app:weak-learnability}.} Therefore, for weak learning on the hypercube, two-layer networks are as good as networks of any depth.
For the converse impossibility result, we apply Theorem~\ref{thm:gd-lower}, recalling that GD is $\Gsignperm$-equivariant by Proposition~\ref{prop:equi}, and noting that $\Gsignperm$-alignment is:
\begin{lemma}\label{lem:technical-symm-bool}
Let $f : \cH_d \to \R$. Then
$\mathcal{C}((f, \cH_d); \Gsignperm) = \max_{k \in [d]}  \binom{d}{k}^{-1} \sum_{\substack{S \subseteq [d] \\ |S| = k}} \hat{f}(S)^2.$
\end{lemma}
\begin{proof} In the following, let $\bs \sim \cH_d$ and $\sigma \sim G_{perm}$, so that $g = (\bs,\sigma) \sim \Gsignperm$. Also let $\bx,\bx' \sim \cH_d$ be independent. For any $h : \cH_d \to \R$, by (a) tensorizing, (b) expanding $f$ in the Fourier basis, (c) the orthogonality relation $\E_{\bs}[\chi_S(\bs)\chi_{S'}(\bs)] = \delta_{S,S'}$, and (d) tensorizing,
\begin{align*}
\E_{g}[\E_{\bx}[f(g(\bx))h(\bx)]^2] &= 
\E_{\sigma,\bs}[\E_{\bx}[f(\bs \odot \sigma(\bx))h(\bx)]^2] \\
&\stackrel{(a)}{=} \E_{\sigma,\bs,\bx,\bx'}[f(\bs \odot \sigma(\bx))f(\bs \odot \sigma(\bx'))h(\bx)h(\bx')] \\
&\stackrel{(b)}{=} \E_{\bx,\bx',\sigma}[\sum_{S,S' \subseteq [d]} \hat{f}(S) \hat{f}(S') h(\bx)h(\bx') \chi_S(\sigma(\bx))\chi_{S'}(\sigma(\bx'))\E_{\bs} [\chi_{S}(\bs)\chi_{S'}(\bs)]] \\
&\stackrel{(c)}{=} \E_{\bx,\bx',\sigma}[\sum_{S \subseteq [d]} \hat{f}(S)^2 h(\bx)h(\bx') \chi_S(\sigma(\bx))\chi_S(\sigma(\bx'))] \\
&\stackrel{(d)}{=} \E_{\sigma}[\sum_{S \subseteq [d]} \hat{f}(S)^2 \E_{\bx}[h(\bx) \chi_S(\sigma(\bx))]^2] \\
&= \sum_{S \subseteq [d]} \hat{f}(S)^2 \E_{\sigma}[\hat{h}(\sigma^{-1}(S))^2] \\
&= \sum_{S \subseteq [d]} \hat{f}(S)^2 \binom{d}{|S|}^{-1} \sum_{S', |S'| = |S|} \hat{h}(S')^2.
\end{align*}
And since $\sum_{S', |S'| = |S|} \hat{h}(S')^2 \leq \|h\|^2_{\Capitalltwo(\cH_d)}$, the supremum over $h$ such that $\|h\|_{\Capitalltwo(\cH_d)} = 1$ is achieved by taking $h(\bx) = \chi_{S}(\bx)$ for some $S$.
\end{proof}
So if the Fourier coefficients of $f$ are negligible for all $S$ s.t.\ $\min(|S|,d-|S|) \leq O(1)$, then the $\Gsignperm$-alignment of $f$ is negligible. By Theorem~\ref{thm:gd-lower}, this means $f$ cannot be learned efficiently. In Appendix~\ref{app:paritymod4} we give a concrete example of a hard function, that was not previously known.

\subsubsection{Functions on sphere, FC networks with i.i.d. Gaussian initialization}

We now study learning a target function on the unit sphere, $f \in \Capitalltwo(\S^{d-1})$, where we take the standard Lebesgue measure on $\S^{d-1}$. A key fact in harmonic analysis is that $\Capitalltwo(\S^{d-1})$ can be written as the direct sum of subspaces spanned by spherical harmonics of each degree  (see, e.g., \cite{hochstadt2012functions}).
$$\Capitalltwo(\S^{d-1}) = \bigoplus_{l=0}^{\infty} \cV_{d,l},$$ 
where  $\cV_{d,l} \subseteq \Capitalltwo(\S^{d-1})$ is the space of degree-$l$ spherical harmonics, which is of dimension $$\dim(\cV_{d,l}) = \frac{2l + d - 2}{l} \binom{l + d - 3}{l-1}.$$
Let $\Pi_{\cV_{d,l}} : \Capitalltwo(\S^{d-1}) \to \cV_{d,l}$ be the projection operator to the space of degree-$l$ spherical harmonics.
In Appendix~\ref{app:sphere-char}, we prove this characterization of weak-learnability for functions on the sphere:
\begin{theorem}\label{thm:sphere-weak-char}
Let $\{f_d\}_{d \in \NN}$ be a family of functions $f_d : \S^{d-1} \to \R$ with $\|f_d\|_{\Capitalltwo(\S^{d-1})} \leq 1$. Then $\{f_d,\S^{d-1}\}$ is efficiently weak-learnable by GD-trained Gaussian-initialized FC networks if and only if there is a constant $C > 0$ such that 
$\sum_{l=0}^{C} \|\Pi_{\cV_{d,l}} f_d\|^2 \geq d^{-C}$.
\end{theorem}

The algorithmic result can again be achieved by two-layer FC networks, and is a consequence of the analysis of the random feature kernel in \cite{ghorbani2021linearized}, which shows that the projection of $f_d$ onto the low-degree spherical harmonics can be efficiently learned. For the  impossibility result, we apply Theorem~\ref{thm:gd-lower}, noting that GD is $\Grot$-equivariant by Proposition~\ref{prop:equi}, and the $\Grot$-alignment is:
\begin{lemma}\label{lem:rotation-sphere-alignment-old}
Let $f \in \Capitalltwo(\S^{d-1})$. Then $\mathcal{C}((f, \S^{d-1}); \Grot) = \max_{l \in \ZZ_{\geq 0}} \|\Pi_{\cV_{d,l}} f\|^2 / \dim(\cV_{d,l})$.
\end{lemma}
\begin{proof}

The $\Grot$-alignment is computed using the representation theory of $\Grot$, specifically the Schur orthogonality theorem (see, e.g., \cite{serre1977linear,knapp1996lie}). For any $l$, the subspace $\cV_{d,l}$ is invariant to action by $\Grot$, meaning that we may define the representation $\Phi_l$ of $\Grot$, which for any $g \in \Grot, f \in \cV_{d,l}$ is given by $\Phi_l(g) : \cV_{d,l} \to \cV_{d,l}$ and $\Phi_l(g) f = f \circ g^{-1}$. Furthermore, $\Phi_l$ is a unitary, irreducible representation, and $\Phi_l$ is not equivalent to $\Phi_{l'}$, for any $l \neq l'$ (see e.g., \cite[Theorem 1]{stanton1990introduction}). Therefore, by the Schur orthogonality relations \cite[Corollary~4.10]{knapp1996lie}, for any $v_1, w_1 \in \cV_{d,l_1}$ and $v_2, w_2 \in \cV_{d,l_2}$, we have
\begin{align}\label{eq:schur-orthogonality-rotations}
\E_{g \sim \Grot}[\<\phi_{l_1}(g) v_1, w_1\>_{\Capitalltwo(\S^{d-1})}& \<\phi_{l_2}(g) v_2, w_2\>_{\Capitalltwo(\S^{d-1})}] \nonumber \\
&= \delta_{l_1 l_2} \<v_1, v_2\>_{\Capitalltwo(\S^{d-1})} \<w_1, w_2\>_{\Capitalltwo(\S^{d-1})} / \dim(\cV_{d,l_1}).
\end{align}

Let $g \sim \Grot$, drawn from the Haar probability measure. For any $h \in \Capitalltwo(\S^{d-1})$ such that $\|h\|^2_{\Capitalltwo(\S^{d-1})} = 1$, by (a) the decomposition of $\Capitalltwo(\S^{d-1})$ into subspaces of spherical harmonics, (b) the $\Grot$-invariance of each subspace $\cV_{d,l}$, and (c) the Schur orthogonality relations in \eqref{eq:schur-orthogonality-rotations}, \begin{align*}
\E_{g}[\<f \circ g, h\>^2_{\Capitalltwo(\S^{d-1})}] &\stackrel{(a)}{=} \sum_{l_1,l_2=0}^{\infty} \E_{g}[\<\Pi_{\cV_{d,l_1}} (f \circ g), \Pi_{\cV_{d,l_1}} h\>_{\Capitalltwo(\S^{d-1})} \<\Pi_{\cV_{d,l_2}} (f \circ g), \Pi_{\cV_{d,l_2}} h\>_{\Capitalltwo(\S^{d-1})}] \\
&\stackrel{(b)}{=} \sum_{l_1,l_2=0}^{\infty} \E_{g}[\<(\Pi_{\cV_{d,l_1}} f) \circ g, \Pi_{\cV_{d,l_1}} h\>_{\Capitalltwo(\S^{d-1})} \<(\Pi_{\cV_{d,l_2}} f) \circ g, \Pi_{\cV_{d,l_2}} h\>_{\Capitalltwo(\S^{d-1})}] \\
&\stackrel{(c)}{=} \sum_{l=0}^{\infty} \frac{1}{\dim(\cV_{d,l})} \|\Pi_{\cV_{d,l}} f\|^2_{\Capitalltwo(\S^{d-1})} \|\Pi_{\cV_{d,l}} h\|_{\Capitalltwo(\S^{d-1})}^2 \\
&\leq \left(\sum_{l=0}^{\infty} \|\Pi_{\cV_{d,l}} h\|_{\Capitalltwo(\S^{d-1})}^2\right) \max_{l \in \ZZ_{\geq 0}} \frac{1}{\dim(\cV_{d,l})} \|\Pi_{\cV_{d,l}} f\|^2_{\Capitalltwo(\S^{d-1})} \\
&= \max_{l \in \ZZ_{\geq 0}} \frac{1}{\dim(\cV_{d,l})} \|\Pi_{\cV_{d,l}} f\|^2_{\Capitalltwo(\S^{d-1})}.
\end{align*}
Let $l^*$ be the optimal value of $l$ in the last line, which is known to exist by the fact that $\|\Pi_{\cV_{d,l}} f\|^2 \leq \|f\|^2$ and $\dim(\cV_{d,l}) \to \infty$ as $l \to \infty$. The inequality is achieved by $h = \Pi_{\cV_{d,l^*}} f / \|\Pi_{\cV_{d,l^*}} f\|$.
\end{proof}

This implies that the $\Grot$-alignment of $f$ is negligible if and only if its projection to the low-order spherical harmonics is negligible. By Theorem~\ref{thm:gd-lower}, this implies the necessity result of Theorem~\ref{thm:sphere-weak-char}.

\subsection{Application: Extending the merged-staircase property necessity result}\label{sec:msp}

In our second application, we study the setting of learning a sparse function on the binary hypercube (a.k.a. a junta) that depends on only $P \leq d$ coordinates of the input $\bx$, i.e., 
    $$f_*(\bx) = h_*(x_1,\ldots,x_P),$$ 
where $h_* : \cH_P \to \R$. The regime of interest to us is when $h_*$ is fixed and $d \to \infty$, representing a hidden signal in a high-dimensional dataset. This setting was studied by \cite{abbe2022merged}, who identified the ``merged-staircase property'' (MSP) as an extension of \cite{abbe2021staircase}. We generalize the MSP below.
\begin{definition}[$l$-MSP] For $l \in \ZZ_{+}$ and $h_* : \cH_P \to \R$, we say that $h_*$ satisfies the merged staircase property with leap $l$ (i.e., $l$-MSP) if its set of nonzero Fourier coefficients $\cS = \{S : \hat{h}_*(S) \neq \emptyset\}$ can be ordered as $\cS = \{S_1,\ldots,S_m\}$ such that for all $i \in [m]$, $|S_i \sm \cup_{j < i} S_j| \leq l$. 
\end{definition}

For example, $h_*(\bx) = x_1 + x_1x_2 + x_1x_2x_3$ satisfies 1-MSP; $h_*(\bx) = x_1x_2 + x_1x_2x_3$ satisfies 2-MSP, but not 1-MSP because of the leap required to learn $x_1x_2$;  similarly $h_*(\bx) = x_1x_2x_3 + x_4$ satisfies 3-MSP but not 2-MSP. If $h_*$ satisfies $l$-MSP for some small $l$, then the function $f_*$ can be learned greedily in an efficient manner, by iteratively discovering the coordinates on which it depends. In \cite{abbe2022merged} it was proved that the 1-MSP property nearly characterized which sparse functions could be $\eps$-learned in $O_{\eps,h_*}(d)$ samples by one-pass SGD training in the mean-field regime.

We prove the MSP necessity result for GD training. On the one hand, our necessity result is for a different training algorithm, GD, which injects noise during training. On the other, our result is much more general since it applies whenever GD is permutation-equivariant, which includes training of FC networks and ResNets of any depth (whereas the necessity result of \cite{abbe2022merged} applies only to two-layer architectures in the mean-field regime). We also generalize the result to any leap $l$.

\begin{theorem}[$l$-MSP necessity]\label{thm:msp-bool}
Let $\fNN(\cdot; \btheta) : \cH_d \to \R$ be an architecture and $\mu_{\btheta} \in \cP(\R^p)$ be an initialization such that GD is $G_{perm}$-equivariant. Let $\btheta^k$ be the random weights after $k$ steps of GD training with noise parameter $\tau > 0$, step size $\eta$, and clipping radius $R$ on the distribution $\cD = \cD(f_*, \cH_d)$. Suppose that $f_*(\bx) = h_*(\bz)$ where $h_* : \cH_P \to \R$ does not satisfy $l$-MSP for some $l \in \ZZ_+$. Then there are constants $C, \eps_0 > 0$ depending on $h_*$ such that
$$\PP_{\btheta^k}[\ell_{\cD}(\btheta^k) \leq \eps_0] \leq \frac{C \eta R}{2\tau} \sqrt{\frac{k}{d^{l+1}}} + \frac{C}{d^{l+1}}.$$
\end{theorem}

The interpretation is that if $h_*$ does not satisfy $l$-MSP, then to learn $f_*$ to better than $\eps_0$ error with constant probability, we need at least $\Omega_{h_*,\eps}(d^{l+1})$ steps of \eqref{GD} on a network with step size $\eta = O_{h_*,\eps}(1)$, clipping radius $R = O_{h_*,\eps}(1)$, and noise level $\tau = \Omega_{h_*,\eps}(1)$. The proof is deferred to Appendix~\ref{app:msp}. It proceeds by first isolating the ``easily-reachable'' coordinates $T \subseteq [P]$, and subtracting their contribution from $f_*$. We then bound $G$-alignment of the resulting function, where $G$ is the permutation group on $[d] \sm T$.

\section{Hardness for learning with SGD}\label{sec:results-sgd}

In this section, for $\gamma > 0$, we let $\cD(f,\mu_{\cX},\gamma) \in \cP(\cX \times \R)$ denote the distribution of $(\bx, f(\bx) + \xi)$ where $\bx \sim \mu_{\cX}$ and $\xi \sim \cN(0,\gamma^2)$ is independent noise.

We show that the equivariance of SGD on certain architectures implies that the function $\fmodeight : \cH_d \to \{0,\ldots,7\}$ given by
\begin{align}\fmodeight(\bx) \equiv \sum_i x_i \pmod{8} \label{eq:mod8}
\end{align}is hard for SGD-trained, i.i.d. symmetrically-initialized FC networks. Our hardness result relies on a cryptographic assumption to prove superpolynomial lower bounds for SGD learning. 
For any $S \subseteq [d]$, let $\chi_S : \cH_d \to \{+1,-1\}$ be the parity function $\chi_S(\bx) = \prod_{i \in S} x_i$.

\begin{definition}\label{def:lpgn}
The learning parities with Gaussian noise, $(d,n,\gamma)$-LPGN, problem is parametrized by $d,n \in \ZZ_{> 0}$ and $\gamma \in \R_{> 0}$. An instance $(S, \bq, (\bx_i,y_i)_{i \in [n]})$ consists of
(i) an unknown subset $S \subseteq [d]$ of size $|S| = \floor{d/2}$, and (ii) a known query vector $\bq \sim \cH_d$, and i.i.d. samples $(\bx_i,y_i)_{i \in [n]} \sim \cD(\chi_S, \cH_d, \gamma)$. The task is to return $\chi_S(\bq) \in \{+1,-1\}$.\footnote{More formally, one would express this as a probabilistic promise problem \cite{alekhnovich2003more}.}
\end{definition}
Our cryptographic assumption is that $\poly(d)$-size circuits cannot succeed on LPGN.
\begin{definition}
Let $\gamma > 0$. We say $\gamma$-LPGN is $\poly(d)$-time solvable if there is a sequence of sample sizes $\{n_d\}_{d \in \NN}$ and circuits $\{\mathcal{A}_d\}_{d \in \NN}$ such that $n_d, \mathrm{size}(\mathcal{A}_d) \leq \poly(d)$, and $\cA_d$ solves
$(d,n_d,\gamma)$-LPGN with success probability at least $9/10$, when inputs are rounded to $\poly(d)$ bits.

\end{definition}
\begin{assumption}\label{ass:lpgn-hardness}Fix $\gamma$. The $\gamma$-LPGN-hardness assumption is: $\gamma$-LPGN is not $\poly(d)$-time solvable.
\end{assumption}

The LPGN problem is the simply standard Learning Parities with Noise problem (LPN) \cite{blum2003noise}, except with Gaussian noise instead of binary classification noise, and we are also promised that $|S| = \floor{d/2}$. In Appendix~\ref{app:lpgn}, we  derive Assumption~\ref{ass:lpgn-hardness} from the standard hardness of LPN. We now state our SGD hardness result.
\begin{theorem}\label{thm:sgd-sum-mod-8-hard}
Let $\{\fNNd, \mu_{\btheta,d}\}_{d \in \N}$ be a family of networks and initializations satisfying Assumption~\ref{ass:noskip} (fully-connected) with i.i.d. symmetric initialization. Let $\gamma > 0$, and let $\{n_d\}$ be sample sizes such that $(\fNNd,\mu_{\btheta,d})$-SGD training on $n_d$ samples from $\cD(\fmodeight,\cH_d,\gamma)$ rounded to $\poly(d)$ bits yields parameters $\btheta_d$ with 
$$\E_{\btheta_{d}}[\|\fmodeight - \fNN(\cdot;\btheta_d)\|^2] \leq 0.0001.$$
Then, under $(\gamma/2)$-LPGN hardness, $(\fNNd,\mu_{\btheta,d})$-SGD on $n_d$ samples cannot run in $\poly(d)$ time.
\end{theorem}
In order to prove Theorem~\ref{thm:sgd-sum-mod-8-hard}, we use the sign-flip equivariance of gradient descent guaranteed by the symmetry in the initialization. A sign-flip equivariant network that learns $\fmodeight(\bx)$ from $\gamma$-noisy samples, is capable of solving the harder problem of learning $\fmodeight(\bx \odot \bs)$ from $\gamma$-noisy samples, where $\bs \in \cH_d$ is an unknown sign-flip vector. However, through an average-case reduction we show that this problem is $(\gamma/2)$-LPGN-hard. Therefore the theorem follows by contradiction.

\section{Discussion}\label{sec:discussion}
The general GD lower bound in Theorem~\ref{thm:gd-lower} and the approach for basing hardness of SGD training on cryptographic assumptions in Theorem~\ref{thm:sgd-sum-mod-8-hard} could be further developed to other settings.

There are limitations of the results to address in future work. First, the GD lower bound requires adding noise to the gradients, which can hinder training. Second, real-world data distributions are typically not invariant to a group of transformations, so the results obtained by this work may not apply. It is open to develop results for distributions that are approximately invariant.

Finally, it is open whether computational lower bounds for SGD/GD training can be shown beyond those implied by equivariance. For example, consider the function $f : \cH_d \to \{+1,-1\}$ that computes the ``full parity'', i.e., the parity of all of the inputs $f(\bx) = \prod_{i=1}^d x_i$. Past work has empirically shown that SGD on FC networks with Gaussian initialization \cite{shalev2017failures, abbe2020poly,nachum2021symmetry} fails to learn this function. Proving this would represent a significant advance, since there is no obvious equivariance that implies that the full parity is hard to learn --- in fact we have shown weak-learnability with symmetric $\Rad(1/2)$ initialization, in which case training is $\Gsignperm$-equivariant.

\section*{Acknowledgements}
We thank Jason Altschuler, Guy Bresler, Elisabetta Cornacchia, Sonia Hashim, Jan Hazla, Hannah Lawrence, Theodor Misiakiewicz, Dheeraj Nagaraj, and Philippe Rigollet for stimulating discussions. We thank the Simons Foundation and the NSF for supporting us through the Collaboration on the Theoretical Foundations of Deep Learning (deepfoundations.ai). This work was done in part while E.B. was visiting the Simons Institute for the Theory of Computing and the Bernoulli Center at EPFL, and was generously supported by Apple with an AI/ML fellowship.

\bibliographystyle{alphaabbr}
\bibliography{bibliography.bib}

\clearpage

\setcounter{tocdepth}{2}
\tableofcontents

\appendix

\section{Lower bound for GD, Proof of Theorem~\ref{thm:gd-lower} and generalization}\label{sec:gd}

We prove a generalization of Theorem~\ref{thm:gd-lower}, which applies to learning a random function $f_* \sim \mu_{\cF}$, where $\mu_{\cF}$ is a distribution over functions in $\Capitalltwo(\mu_{\cX})$. In order to state the theorem, let us first define the $G$-alignment for a distribution over functions:
\begin{definition}\label{def:g-alignment-gen}
Let $G$ be a compact group that acts on $\cX$, and let $\mu_{\cX} \in \cP(\cX)$ be a $G$-invariant distribution. For any distribution of functions $\mu_{\cF} \in \cP(\Capitalltwo(\mu_{\cX}))$, we define the $G$-alignment as:
\begin{align*}
\mathcal{C}((\mu_{\cF}, \mu_{\cX}); G) = \sup_{h} \E_{g \sim \mu_G}[\E_{f \sim \mu_{\cF}}[\<f \circ g, h\>_{\Capitalltwo(\mu_{\cX})}^2]],
\end{align*}
where $\mu_G$ is the Haar probability measure over $G$, and the supremum is over $h \in \Capitalltwo(\mu_{\cX})$ such that $\|h\|^2 = 1$.
\end{definition}

This is a generalization of the $G$-alignment of Definition~\ref{def:g-alignment}, since we can take $\mu_{\cF}$ to be the probability distribution that has all mass on a deterministic function $f_*$. We use it to prove the following generalization of Theorem~\ref{thm:gd-lower}:

\begin{theorem}[GD lower bound for distribution of functions, using $G$-alignment]\label{thm:gd-lower-general}
Let $G$ be a compact group that acts on $\cX$, and let $\fNN(\cdot; \btheta) : \cX \to \R$ be an architecture and $\mu_{\btheta} \in \cP(\R^p)$ be an initialization distribution, such that GD is $G$-equivariant.

Fix any $G$-invariant distribution $\mu_{\cX} \in \cP(\cX)$. For any $f \in \Capitalltwo(\mu_{\cX})$, let $\btheta_{f}^k$ be the random weights after $k$ time-steps of \eqref{GD} training with noise parameter $\tau > 0$, step size $\eta > 0$, and clipping radius $R > 0$ on the distribution $\cD_{f} = \cD(f, \mu_{\cX})$. Fix any baseline function $\alpha \in \Capitalltwo(\mu_{\cX})$ such that $\alpha \circ g = \alpha$ for all $g \in G$, and let $\mu_{\cF} \in \cP(\Capitalltwo(\mu_{\cX}))$ be any distribution of target functions. Then, for any $\eps > 0$,
\begin{align*}
\PP_{f \sim \mu_{\cF}, \btheta^k_{f}}[\ell_{\cD_{f}}(\btheta_{f}^k) \leq \|f - \alpha\|^2_{\Capitalltwo(\mu_{\cX})} - \eps] \leq \frac{\eta R \sqrt{k \mathcal{C}}}{2 \tau} + \frac{ \mathcal{C}}{\eps},
\end{align*}
where $\mathcal{C} = \mathcal{C}((\bar{\mu}_{\cF}, \mu_{\cX}); G)$, and $\bar{\mu}_{\cF}$ is the distribution of $f - \alpha$ for $f \sim \mu_{\cF}$.
\end{theorem}

Theorem~\ref{thm:gd-lower} is the special case in which $\mu_{\cF}$ has all probability mass on one atom $f_*$.

\subsection{Proof of Theorem~\ref{thm:gd-lower-general}}

We derive Theorem~\ref{thm:gd-lower-general} from the following theorem, which is the same bound but without the $G$-equivariance assumption (think of $G$ as being the trivial group). We first define the alignment of a distribution of functions:
\begin{definition}
Let $\mu_{\cX} \in \cP(\cX)$. For any distribution of functions $\mu_{\cF} \in \cP(\Capitalltwo(\mu_{\cX}))$, we define:
\begin{align*}
\mathcal{C}(\mu_{\cF}, \mu_{\cX}) = \sup_{h} \E_{f \sim \mu_{\cF}}[\<f, h\>^2],
\end{align*}
where the supremum is over $h \in \Capitalltwo(\mu_{\cX})$ such that $\|h\|^2 = 1$.
\end{definition}

\begin{theorem}[GD lower bound for distribution of functions]\label{thm:gd-lower-helper}
Let $\fNN(\cdot; \btheta) : \cX \to \R$ be an architecture, and let $\mu_{\btheta} \in \cP(\R^p)$ be an initialization distribution.

Fix $\mu_{\cX} \in \cP(\cX)$. For any $f \in \Capitalltwo(\mu_{\cX})$, let $\btheta_f^k$ be the random weights after $k$ time-steps of \eqref{GD} training with noise parameter $\tau > 0$, step size $\eta > 0$, and clipping radius $R > 0$ on the distribution $\cD_f = \cD(f, \mu_{\cX})$. Then, for any $\alpha \in \Capitalltwo(\mu_{\cX})$ any $\mu_{\cF} \in \cP(\Capitalltwo(\mu_{\cX}))$, and any $\eps > 0$,
\begin{align*}
\PP_{f \sim \mu_{\cF}, \btheta_f^k} [\ell_{\cD_{f}}(\btheta_f^k) \leq \|f - \alpha\|^2_{\Capitalltwo(\mu_{\cX})} - \eps] \leq \frac{\eta R \sqrt{k \mathcal{C}}}{2 \tau} + \frac{\mathcal{C}}{\eps},
\end{align*}
where $\mathcal{C} = \mathcal{C}(\bar{\mu}_{\cF},\mu_{\cX})$, and $\bar{\mu}_{\cF}$ is the distribution of $f - \alpha$ for $f \sim \mu_{\cF}$.
\end{theorem}

We defer the proof of this theorem, and first use it to prove Theorem~\ref{thm:gd-lower-general}.

\begin{proof}[Proof of Theorem~\ref{thm:gd-lower-general}]
Let $\mu_G$ be the Haar probability measure on $G$. Define the distribution of functions $\nu_{\cF} \in \cP(\Capitalltwo(\mu_{\cX}))$ as the distribution of $f \circ g \in \Capitalltwo(\mu_{\cX})$, where $f \sim \mu_{\cF}$ and $g \sim \mu_{G}$, independently. Notice that $f \circ g$ is in $\Capitalltwo(\mu_{\cX})$ since $\mu_{\cX}$ is $G$-invariant and $f \in \Capitalltwo(\mu_{\cX})$, so this is well defined.

Now define $\bar{\mu}_{\cF}$ and $\bar{\nu}_{\cF}$ to be the distribution of $f - \alpha$ for $f \sim \mu_{\cF}$ and $f \sim \nu_{\cF}$, respectively. Since $\alpha \circ g = \alpha$ for all $g \in G$, we have
\begin{align*}
\mathcal{C}((\bar{\mu}_{\cF}, \mu_{\cX}); G) &= \sup_h \E_{g \sim \mu_G}[\E_{f \sim \mu_{\cF}}[\<(f - \alpha) \circ g, h\>^2]] = \sup_h \E_{g \sim \mu_G}[\E_{f \sim \mu_{\cF}}[\<(f  \circ g) - \alpha, h\>^2]] \\
&= \sup_h \E_{f \sim \bar{\nu}_{\cF}}[\<f, h\>^2] = \mathcal{C}(\bar{\nu}_{\cF}, \mu_{\cX})\,.
\end{align*}
We conclude by Theorem~\ref{thm:gd-lower-helper} that
\begin{align*}
\PP_{f \sim \mu_{\cF}, g \sim \mu_G, \btheta_{f \circ g}^k}[\ell_{\cD_{f \circ g}}(\btheta_{f \circ g}^k)  \leq \|f \circ g - \alpha\|^2 - \eps] \leq \frac{\eta R \sqrt{k\mathcal{C}}}{2\tau} + \frac{\mathcal{C}}{\eps},
\end{align*}
for $\mathcal{C} = \mathcal{C}(\bar{\mu}_{\cF}, \mu_{\cX}, G)$. To derive Theorem~\ref{thm:gd-lower-general}, note for any $f \in \Capitalltwo(\mu_{\cX})$ and $g \in G$, by the $G$-equivariance of GD training we have $\ell_{\cD_{f \circ g}}(\btheta_{f \circ g}^k) \stackrel{d}{=} \ell_{\cD_f}(\btheta_f^k)$. Finally, conclude by noting $\|f \circ g - \alpha\|_{\Capitalltwo(\mu_{\cX})}^2 = \|(f - \alpha) \circ g\|_{\Capitalltwo(\mu_{\cX})}^2 = \|f - \alpha\|_{\Capitalltwo(\mu_{\cX})}^2$, where we first use the $G$-invariance of $\alpha$ and then that of $\mu_{\cX}$. 
\end{proof}

\subsection{Proof of Theorem~\ref{thm:gd-lower-helper}}

The proof of Theorem~\ref{thm:gd-lower-helper} is a variation on the junk flow argument of \cite{abbe2020poly}. There are two significant differences. First, the junk data distribution is not chosen so that labels are independent of data, but rather chosen so that the labels are given by the function $\alpha$. This allows us to prove lower bounds beyond weak learning (cf. allowing the merged-staircase property necessity result of Section~\ref{sec:msp}). Second, instead of using cross-predictability, we use a tighter bound based on the quantity $\mathcal{C}(\bar\mu_{\cF}, \mu_{\cX})$. This allows a $\Omega(d^{l+1})$-lower bound instead of a $\Omega(d^{(l+1)/2})$-lower bound for functions that are not $l$-MSP in Section~\ref{sec:msp}. However, for this tighter bound during training we need to clip the gradients of the neural network instead of the gradients of the loss as in \cite{abbe2020poly}. In Appendix~\ref{app:cross-pred}, we show how to recover the bound of \cite{abbe2020poly} based on cross-predictability if we instead clip the gradients of the loss.

For the analysis, we define the following gradient descent trajectories:
\begin{itemize}
    \item \textit{GD trajectory on $\cD_f$}: for any $f : \cX \to \R$, we let $\btheta^0_{f},\ldots,\btheta^k_{f}$ be the trajectory of \eqref{GD} on the data distribution $\cD_f = \cD(f,\mu_{\cX}) \in \cP(\cX \times \R)$. I.e., we initialize $\btheta^0 \sim \mu_{\btheta}$, and update
    \begin{align*}\btheta_{f}^{k+1} = \btheta_{f}^k - \eta \bg_{\cD_{f}}(\btheta^k_f) + \bxi^k\,,\qquad\mbox{ where } \bxi^k \sim \cN(0,\tau^2 \bI)\,.\end{align*}
    
    \item \textit{Junk GD trajectory}: define $\Djunk = \cD(\alpha,\mu_{\cX})$.
    We let $\btheta_{\alpha}^0,\ldots,\btheta_{\alpha}^k$ be the trajectory of \eqref{GD} on $\Djunk$, which we call the \textit{junk trajectory}. I.e., we initialize $\btheta_{\alpha}^0 \sim \mu_{\btheta}$, and update
    \begin{align*}\btheta_{\alpha}^{k+1} = \btheta_{\alpha}^k - \eta \bg_{\Djunk}(\btheta_{\alpha}^k) + \tilde\bxi^k\,,\qquad\mbox{ where } \tilde\bxi^k \sim \cN(0,\tau^2 \bI)\,.\end{align*}
\end{itemize}

We now prove that the junk trajectory $\btheta_{\alpha}^k$ stays close to the trajectory $\btheta^k_f$, for most functions $f \sim \mu_{\cF}$. The bound depends on the quantity $\mathcal{C}(\bar\mu_{\cF}, \mu_{\cX})$.
\begin{lemma}[Junk trajectory close to most GD trajectories]\label{lem:couple-gd-equi-with-junk}
Under the assumptions of Theorem~\ref{thm:gd-lower-helper},
\begin{align*}
\E_{f \sim \mu_{\cF}}[\TV(\cL(\btheta_{\alpha}^k),\cL(\btheta^k_{f}))] \leq \frac{\eta R}{2\tau} \sqrt{k \mathcal{C}(\bar{\mu}_{\cF}, \mu_{\cX})}.
\end{align*}
\end{lemma}

Finally, we show that the junk GD trajectory is not correlated with most random $f \sim \mu_{\cF}$. This bound again depends on the quantity $\mathcal{C}(\bar\mu_{\cF}, \mu_{\cX})$.
\begin{lemma}[Junk trajectory does not learn]\label{lem:junk-stuck}
Under the assumptions of Theorem~\ref{thm:gd-lower-helper}, for any $\eps > 0$,
$$\PP_{f \sim \mu_{\cF}, \btheta_{\alpha}^k}[\ell_{\cD_f}(\btheta_{\alpha}^k) \leq \|f-\alpha\|^2- \eps] \leq \mathcal{C}(\bar\mu_{\cF}, \mu_{\cX}) / \eps.$$
\end{lemma}

Combining the above two lemmas, we prove the theorem.
\begin{proof}[Proof of Theorem~\ref{thm:gd-lower-helper}]
By Lemma~\ref{lem:couple-gd-equi-with-junk}, we can couple $\btheta_{\alpha}^k$ with $\btheta^k_f$ so that $$\PP_{f \sim \mu_{\cF}, \btheta_{\alpha}^k, \btheta^k_f}[\btheta_{\alpha}^k \neq \btheta^k_f] \leq \frac{\eta R}{2 \tau} \sqrt{k \mathcal{C}(\bar\mu_{\cF}, \mu_{\cX})}.$$ So, for any $t > 0$,
\begin{align*}
\PP_{f \sim \mu_{\cF}, \btheta^k_f}[\ell_{\cD_f}(\btheta^k_f) \leq t] \leq \PP_{f \sim \mu_{\cF}, \btheta_{\alpha}^k}[\ell_{\cD_f}(\btheta_{\alpha}^k) \leq t] + \frac{\eta R}{2 \tau} \sqrt{k \mathcal{C}(\bar\mu_{\cF}, \mu_{\cX})},
\end{align*}
and the theorem follows by Lemma~\ref{lem:junk-stuck}.
\end{proof}

\subsubsection{Proofs of auxiliary lemmas}

\begin{proof}[Proof of Lemma~\ref{lem:couple-gd-equi-with-junk}]
For brevity, write $\btheta_{\alpha}^{\leq k} = (\btheta_{\alpha}^0,\ldots,\btheta_{\alpha}^k)$, and similarly for $\btheta^{\leq k}_{f}$. We bound the KL-divergence between the junk trajectory $\btheta_{\alpha}^{\leq k}$ and the trajectory $\btheta^{\leq k}_{f}$ by (a) using the chain rule for KL-divergence, (b) the fact that $\cL(\btheta_{\alpha}^0) = \cL(\btheta^0_{f}) = \mu_{\btheta}$, (c) the Markov property of GD training, (d) the definition of the update step \eqref{GD}, and (e) the KL divergence between two Gaussians:
\begin{align}
\KL(&\cL(\btheta_{\alpha}^{\leq k}) || \cL(\btheta^{\leq k}_{f})) \\
&\stackrel{(a)}{=} \KL(\cL(\btheta_{\alpha}^0) || \cL(\btheta^0_{f})) + \sum_{k'=1}^k \KL(\cL(\btheta_{\alpha}^{k'} | \btheta_{\alpha}^{\leq k'-1}) || \cL(\btheta^{k'}_{f} | \btheta^{\leq k'-1}_{f})) \nonumber \\
&\stackrel{(b)}{=} \sum_{k'=1}^k \KL(\cL(\btheta_{\alpha}^{k'} | \btheta_{\alpha}^{\leq k'-1}) || \cL(\btheta^{k'}_{f} | \btheta^{\leq k'-1}_{f})) \nonumber  \\
&\stackrel{(c)}{=} \sum_{k'=1}^k \KL(\cL(\btheta_{\alpha}^{k'} | \btheta_{\alpha}^{k'-1}) || \cL(\btheta^{k'}_{f} | \btheta^{k'-1}_{f})) \nonumber  \\
&\stackrel{(d)}{=} \sum_{k'=1}^k \E_{\btheta' \sim \cL(\btheta_{\alpha}^{k'-1})} [\KL(\cN(\btheta' - \eta \bg_{\Djunk}(\btheta'),\tau^2 \bI) || \cN(\btheta' - \eta \bg_{\cD_{f}}(\btheta'),\tau^2 \bI))] \nonumber  \\
&\stackrel{(e)}{=} \sum_{k'=1}^k \frac{\eta^2}{2\tau^2} \E_{\btheta' \sim \cL(\btheta_{\alpha}^{k'-1})} [\|\bg_{\Djunk}(\btheta') - \bg_{\cD_{f}}(\btheta')\|^2]. \label{eq:KL-intermediate-step}
\end{align}

In order to analyze this, let us simplify the following quantity. Throughout $\bx,\bx' \sim \mu_{\cX}$ are i.i.d. We use that (a) $\Djunk$ and $\cD_{f}$ have the same marginal distribution of $\bx \sim \mu_{\cX}$, (b) the definitions of $\Djunk$ and $\cD_{f}$,
\begin{align*}
\bg_{\Djunk}(\btheta') &- \bg_{\cD_{f}}(\btheta') \\
&= -\E_{(\bx,y) \sim \Djunk}[(y - \fNN(\bx;\btheta')) \Pi_{B(0,R)} \nabla_{\btheta} \fNN(\bx;\btheta')] \\
&\quad\quad + \E_{(\bx,y) \sim \cD_{f}}[(y - \fNN(\bx;\btheta'))\Pi_{B(0,R)}\nabla_{\btheta} \fNN(\bx;\btheta')] \\
&\stackrel{(a)}{=} -\E_{(\bx,y) \sim \Djunk}[y\Pi_{B(0,R)}\nabla_{\btheta} \fNN(\bx;\btheta')] + \E_{(\bx,y) \sim \cD_{f}}[y \Pi_{B(0,R)}\nabla_{\btheta} \fNN(\bx;\btheta')] \\
&\stackrel{(b)}{=} -\E_{\bx }[\alpha(\bx)\Pi_{B(0,R)}\nabla_{\btheta} \fNN(\bx;\btheta')] + \E_{\bx }[f(\bx) \Pi_{B(0,R)}\nabla_{\btheta} \fNN(\bx;\btheta')] \\
&= \E_{\bx }[(f(\bx) - \alpha(\bx))\Pi_{B(0,R)}\nabla_{\btheta} \fNN(\bx;\btheta')].
\end{align*}
Now let us draw $f \sim \mu_{\cF}$, and bound the expected KL divergence of $\btheta_{\alpha}^{\leq k}$ with $\btheta^{\leq k}_{f}$. By (a) plugging the above equation into \eqref{eq:KL-intermediate-step}, (b) using independence of $\btheta_{\alpha}^{k'-1}$ from $f$, (c) using the definition of $\mathcal{C}(\bar\mu_{\cF},\mu_{\cX})$   to bound each coordinate, and (d) using the fact that $\Pi_{B(0,R)}$ is projection to the ball of radius $R$,
\begin{align}
\E_{f}[\KL(\cL(\btheta_{\alpha}^{\leq k}) || \cL(\btheta^{\leq k}_{f}))] &\stackrel{(a)}{=} \E_{f}[\sum_{k'=1}^k \frac{\eta^2}{2\tau^2} \E_{\btheta' \sim \cL(\btheta_{\alpha}^{k'-1})} [\|\E_{\bx }[(f(\bx) - \alpha(\bx))\Pi_{B(0,R)}\nabla_{\btheta} \fNN(\bx;\btheta')]\|^2]] \nonumber \\
&\stackrel{(b)}{=} \sum_{k'=1}^k \frac{\eta^2}{2\tau^2} \E_{\btheta' \sim \cL(\btheta_{\alpha}^{k'-1})} [\E_{f}[\|\E_{\bx }[(f(\bx) - \alpha(\bx))\Pi_{B(0,R)}\nabla_{\btheta} \fNN(\bx;\btheta')]\|^2]] \nonumber \\
&\stackrel{(c)}{\leq} \sum_{k'=1}^k \frac{\eta^2}{2\tau^2} \E_{\btheta' \sim \cL(\btheta_{\alpha}^{k'-1})} [\mathcal{C}(\bar\mu_{\cF},\mu_{\cX}) \E_{\bx }[\|\Pi_{B(0,R)}\nabla_{\btheta} \fNN(\bx;\btheta')\|^2]] \nonumber \\
&\stackrel{(d)}{\leq} \frac{k \eta^2  R^2}{2\tau^2} \mathcal{C}(\bar\mu_{\cF},\mu_{\cX}). \label{eq:KL-bound}
\end{align}
Finally we apply (a) the data processing inequality for total variation distance, (b) Pinsker's inequality, (c) Jensen's inequality, and (d) the bound in \eqref{eq:KL-bound}:
\begin{align*}
\E_{f}[\TV(\cL(\btheta_{\alpha}^k),\cL(\btheta^k_{f}))] &\stackrel{(a)}{\leq} \E_{f}[\TV(\cL(\btheta_{\alpha}^{\leq k}),\cL(\btheta^{\leq k}_{f}))] \\
&\stackrel{(b)}{\leq} \E_{f}[\sqrt{\frac{1}{2}\KL(\cL(\btheta_{\alpha}^{\leq k}) || \cL(\btheta^{\leq k}_{f}))}] \\
&\stackrel{(c)}{\leq} \sqrt{\E_{f}[\frac{1}{2}\KL(\cL(\btheta_{\alpha}^{\leq k}) || \cL(\btheta^{\leq k}_{f}))]} \\
&\stackrel{(d)}{\leq} \frac{\eta L}{2\tau} \sqrt{k \mathcal{C}(\bar\mu_{\cF};\mu_{\cX})}.
\end{align*}

\end{proof}

Finally, we prove Lemma~\ref{lem:junk-stuck}, which is the last remaining lemma.

\begin{proof}
Define $\rho = \E_{\bx}[(f(\bx) - \alpha(\bx)) \fNN(\bx; \btheta_{\alpha}^k)]$. Recall the junk trajectory $\btheta_{\alpha}^k$ is drawn independently of $f \sim \mu_{\cF}$, so, by definition of $\mathcal{C}(\bar\mu_{\cF},\mu_{\cX})$,
\begin{align*}
\E_{f}[\rho^2] \leq \mathcal{C}(\bar\mu_{\cF},\mu_{\cX}) \E_{\bx}[\fNN(\bx; \btheta_{\alpha}^k)^2].
\end{align*}
Let $E$ be the event that $\rho^2 \leq \eps\E_{\bx}[\fNN(\bx; \btheta_{\alpha}^k)^2]$. By a Markov bound, $\PP[E] \geq 1 - \mathcal{C}(\bar\mu_{\cF},\mu_{\cX}) / \eps$. Finally, under event $E$ we have
\begin{align*}
\ell_{\cD_f}(\btheta_{\alpha}^k) &= \E_{\bx}[(f(\bx) - \fNN(\bx; \btheta_{\alpha}^k))^2] \\
&= \E_{\bx}[(f(\bx) - \alpha(\bx))^2] - 2 \E_{\bx}[(f(\bx) - \alpha(\bx)) (\fNN(\bx; \btheta_{\alpha}^k) - \alpha(\bx))] + \E_{\bx}[(\fNN(\bx; \btheta_{\alpha}^k) - \alpha(\bx))^2] \\
&= \E_{\bx}[(f(\bx) - \alpha(\bx))^2] - 2 \rho + \E_{\bx}[(\fNN(\bx; \btheta_{\alpha}^k) - \alpha(\bx))^2] \\
&\geq \E_{\bx}[(f(\bx) - \alpha(\bx))^2] - 2 \sqrt{\eps} \sqrt{\E_{\bx}[(\fNN(\bx; \btheta_{\alpha}^k) - \alpha(\bx))^2]} + \E_{\bx}[(\fNN(\bx; \btheta_{\alpha}^k) - \alpha(\bx))^2] \\
&\geq \E_{\bx}[(f(\bx) - \alpha(\bx))^2] - \eps \\
&= \|f - \alpha\|_{\Capitalltwo(\mu_{\cX})}^2 - \eps
\end{align*}
where in the last line we optimize over the quantity $\E_{\bx}[(\fNN(\bx; \btheta_{\alpha}^k) - \alpha(\bx))^2]$.
\end{proof}

\subsection{Remark: relation to bound based on cross-predictability, and efficiently verifying $G$-alignment is small}\label{app:cross-pred}

In \cite{abbe2020poly}, a similar bound to Theorem~\ref{thm:gd-lower-helper} was proved. The first main difference is that bound of \cite{abbe2020poly} bound applied only to learning functions with binary output alphabet, $\{+1,-1\}$. The second difference is that \cite{abbe2020poly} clips the gradient of the loss instead of clipping the gradient of the network. The third difference is that the bound of \cite{abbe2020poly} was in terms of the cross-predictability, instead of the $G$-alignment of Definition~\ref{def:g-alignment-gen}:

\begin{definition}[Cross-predictability]
For any distribution over the inputs $\mu_{\cX} \in \cP(\cX)$ and any distribution over functions $\mu_{\cF} \in \cP(\Capitalltwo(\mu_{\cX}))$, the cross-predictability is
\begin{align*}
    \mathcal{CP}(\mu_{\cF},\mu_{\cX}) = \E_{f,f' \sim \mu_{\cF}}[\<f,f'\>^2_{\Capitalltwo(\mu_{\cX})}].
\end{align*}
\end{definition}
Nevertheless, we show that the cross-predictability is an upper bound on the alignment.
\begin{lemma}\label{lem:alignment-cross-pred} For any $\mu_{\cX}$ and $\mu_{\cF}$, we have
$\mathcal{C}(\mu_{\cF}, \mu_{\cX}) \leq \sqrt{\mathcal{CP}(\mu_{\cF},\mu_{\cX})}$.
\end{lemma}
\begin{proof}
For any $h  \in \Capitalltwo(\mu_{\cX})$, we show
$\E_{f \sim \mu_{\cF}}[\<f, h\>^2] \leq \sqrt{\mathcal{CP}(\mu_{\cF}, \mu_{\cX})} \|h\|_{\Capitalltwo(\mu_{\cX})}^2$. In the following, let $f,f' \sim \mu_{\cF}$ and $\bx,\bx' \sim \mu_{\cX}$ be independent. By (a) a tensorization trick for $\bx$, (b) Cauchy-Schwarz, (c) the tensorization trick for $f$, and (d) reversing the tensorization trick for $\bx$,
\begin{align*}
\E_{f \sim \mu_{\cF}}[\<f, h\>^2] &= \E_{f}[\E_{\bx}[f(\bx)h(\bx)]^2] \\
&\stackrel{(a)}{=} \E_{f,\bx,\bx'}[f(\bx)f(\bx')h(\bx)h(\bx')] \\
&= \E_{\bx,\bx'}[\E_f[f(\bx)f(\bx')]h(\bx)h(\bx')] \\
&\stackrel{(b)}{\leq} \sqrt{\E_{\bx,\bx'}[\E_f[f(\bx)f(\bx')]^2]}\sqrt{\E_{\bx,\bx'}[h(\bx)^2h(\bx')^2]} \\
&= \sqrt{\E_{\bx,\bx'}[\E_f[f(\bx)f(\bx')]^2]}\E_{\bx}[h(\bx)^2] \\
&\stackrel{(c)}{=} \sqrt{\E_{\bx,\bx',f,f'}[f(\bx)f(\bx')f'(\bx)f'(\bx')]}\E_{\bx}[h(\bx)^2] \\
&\stackrel{(d)}{=} \sqrt{\E_{f,f'}[\E_{\bx}[f(\bx)f'(\bx)]^2]}\E_{\bx}[h(\bx)^2] \\
&= \sqrt{\mathcal{CP}(\mu_{\cF},\mu_{\cX})} \|h\|^2_{\Capitalltwo(\mu_{\cX})}.
\end{align*}
\end{proof}

Interestingly, the above lemma provides an algorithmically-efficient way to verify that the $G$-alignment of a function is small. Namely, the ``$G$-cross-predictability'', is an upper bound on the $G$-alignment.

\begin{corollary}[Efficiently verifying $G$-alignment is small]\label{cor:galignmentvsgcp}
For any compact group $G$, any $G$-invariant distribution $\mu_{\cX} \in \cP(\cX)$, and any $\mu_{\cF} \in \cP(\Capitalltwo(\mu_{\cX}))$, we have
\begin{align*}
\mathcal{C}((\mu_{\cF},\mu_{\cX}); G) \leq \sqrt{\mathcal{CP}((\mu_{\cF}, \mu_X); G)} := \sqrt{\E_{f,f' \sim \mu_{\cF}, g,g' \sim \mu_G}[\E_{\bx \sim \mu_{\cX}}[f(g(\bx))f'(g'(\bx))]^2]}
\end{align*}
\end{corollary}
The right-hand-side in this inequality can be approximated efficiently by taking a large enough empirical sample of $f,f'$ and $g,g'$ and $\bx$. By McDiarmid's inequality, the plug-in empirical estimate will concentrate well for large enough number of samples assuming that the functions $f \sim \mu_{\cF}$ are bounded. Thus, given a distribution of functions $\mu_{\cF}$ and covariates $\mu_{\cX}$, one can empirically show that no $G$-equivariant network can learn them.

However, Corollary~\ref{cor:galignmentvsgcp} is loose, which is why we prove our bounds in terms of the $G$-alignment instead of the $G$-cross-predictability, and means that we can obtain tighter bounds than \cite{abbe2020poly} for learning sparse parities (although with a different scheme for clipping the gradients).
\begin{lemma}
Let $\mu_{\cX} = \mathrm{Unif}[\cH_d]$. Let $k \in \{0,\ldots,d\}$, and let $\mu_{\cF}$ be the uniform distribution on $\{\chi_S : S \subseteq [d], |S| = k\}$, where $\chi_S : \cH_d \to \{+1,-1\}$ is given by $\chi_S(\bx) = \prod_{i \in S} x_i$.
Then
\begin{align*}
\binom{d}{k}^{-1} = \mathcal{C}(\mu_{\cF}, \mu_{\cX}) \ll \sqrt{\mathcal{CP}(\mu_{\cF}, \mu_{\cX})} = \binom{d}{k}^{-1/2}
\end{align*}
\end{lemma}
\begin{proof}
Let $S,S' \subseteq [d]$ be independent subsets of size $|S| = |S'| = k$. The cross-predictability is
\begin{align*}
\mathcal{CP}(\mu_{\cF}, \mu_{\cX}) = \E_{S,S'}[\<\chi_S, \chi_{S'}\>^2_{\Capitalltwo(\cH_d)}] = \E_{S,S'}[\delta_{S,S'}] = \binom{d}{k}^{-1}.
\end{align*}
On the other hand, the alignment of $\mu_{\cF}$ is:
\begin{align*}
\mathcal{C}(\mu_{\cF}, \mu_{\cX}) = \sup_{h} \E_{S}[\<\chi_S, h\>^2] = \sup_h \binom{d}{k}^{-1} \sum_{S'' \subseteq [d], |S''| = k} \hat{h}(S'')^2 \leq \binom{d}{k}^{-1},
\end{align*}
where the last line is by Parseval's theorem.
\end{proof}

\subsection{Remark: alternative proof using the statistical query lens}\label{app:sq}
It is possible to prove the impossibility results for learning with GD by using arguments from the Statistical Query literature (see e.g., \cite{kearns1998efficient,reyzin2020statistical}). In particular, \eqref{GD} fits under the Correlational Statistical Query (CSQ) model of computation because it accesses the data distribution $\cD$ only through noisy evaluations of $\bg_{\cD}(\btheta)$, where $$\bg_{\cD}(\btheta) = \E_{(\bx,y) \sim \cD}[y (\Pi_{B(0,R)} \nabla_{\btheta}\fNN(\bx;\btheta))] - \E_{(\bx,y) \sim \cD}[\fNN(\bx;\btheta) (\Pi_{B(0,R)}  \nabla_{\btheta}\fNN(\bx;\btheta))],$$ for different $\btheta^k$. The left-hand term is an expectation over the distribution $\cD$ of a bounded function times $y$, since $\|\Pi_{B(0,R)} \nabla_{\btheta}\fNN(\bx;\btheta)\|^2 \leq R^2$. And the right-hand term is known since we may assume the distribution $\mu_{\cX}$ is known. Therefore each noisy evaluation of $\bg_{\cD}(\btheta)$ can be made with one CSQ \cite{ben1990learning,bshouty2002using}.\footnote{Although correlational statistical queries are typically defined with adversarially-chosen noise, analogous correlational statistical query lower bounds can be proved with random noise (see, e.g., Section~3.2 of \cite{boix2020}).} 

We provide a sketch of the proof of the impossibility result using the CSQ formalism. By the $G$-equivariance of the GD algorithm, if GD can learn $f \in \Capitalltwo(\mu_{\cX})$ via statistical queries on the distribution $\cD_f$, then for any $g \in G$ it can learn any $f \circ g$ via statistical queries on the distribution $\cD_{f \circ g}$. But, if $\{\cD_{f \circ g}\}_{g \in G}$ has high CSQ dimension then this is impossible, and therefore this proves that GD cannot efficiently learn the function $f$. The high CSQ dimension can ensured if the $G$-alignment of Definition~\ref{def:g-alignment} is small. For further comparison between lower bounds based on CSQ and lower bounds based on a junk-flow argument, see \cite{abbe2020poly}.

\section{Characterization of weak-learnability by GD}\label{app:weak-learnability}

In this section, we give the deferred proofs from Section~\ref{sec:characterize}, characterizing weak-learnability for functions on the Boolean hypercube and functions on the unit sphere by FC networks with i.i.d. symmetric and i.i.d. initialization, respectively.

\subsection{Functions on Boolean hypercube: proof of Theorem~\ref{thm:bool-weak-char}}\label{app:bool-char}
We first prove Theorem~\ref{thm:bool-weak-char}, which concerns functions on the hypercube, learned by FC networks with i.i.d. symmetric initialization. We show the impossibility and achievability results separately.

Recall that for functions $f : \cH_d \to \R$ we may write $$f(\bx) = \sum_{S \subseteq [d]} \hat{f}(S) \chi_S(\bx),$$ where for each
 $S \subseteq [d]$, we define the monomial $\chi_S : \cH_d \to \{+1,-1\}$ by $$\chi_S(\bx) = \prod_{i \in S} x_i$$ and  the Fourier coefficient $$\hat{f}(S) = \E_{\bx \sim \cH_d}[f(\bx)\chi_S(\bx)].$$ These monomial functions form an orthogonal basis over the Boolean hypercube, which we will use throughout this section, i.e., for any $S,S' \subseteq [d]$, $$\E_{\bx \sim \cH_d}[\chi_S(\bx)\chi_{S'}(\bx)] = \delta_{S,S'}.$$

\subsubsection{Proof of impossibility result of Theorem~\ref{thm:bool-weak-char}}
The impossibility result is proved using Theorem~\ref{thm:gd-lower}, and the fact from Proposition~\ref{prop:equi} that GD-training of FC networks with i.i.d. symmetric initialization is $\Gsignperm$-equivariant. First, recall the computation of the $\Gsignperm$-alignment.
\begin{lemma}[Restatement of Lemma~\ref{lem:technical-symm-bool}]\label{lem:technical-symm-bool-restated}
For any $f : \cH_d \to \R$,
\begin{align*}
\mathcal{C}((f, \cH_d); \Gsignperm) = \max_{k \in [d]}  \binom{d}{k}^{-1} \sum_{\substack{S \subseteq [d] \\ |S| = k}} \hat{f}(S)^2.
\end{align*}
\end{lemma}

\begin{remark}[Alternative proof of Lemma~\ref{lem:technical-symm-bool-restated} using the Schur orthogonality theorem]
The same representation theory approach used to prove Lemma~\ref{lem:rotation-sphere-alignment-old} for the $\Grot$-alignment could have been used to prove Lemma~\ref{lem:technical-symm-bool-restated} for $\Gsignperm$, instead of the ad hoc calculation. One would decompose the functions over the Boolean hypercube as $\Capitalltwo(\cH_d) = \oplus_{l=0}^{d} \cW_{d,l}$, where $\cW_{d,l}$ is the subspace of homogeneous degree-$l$ polynomials. The representations of $\Gsignperm$ given by $\Phi_l(g) : \cW_{d,l} \to \cW_{d,l}$ mapping $\Phi_l(g)f = f \circ g^{-1}$ are all irreducible and inequivalent for distinct $l \neq l'$, and $\dim(\cW_{d,l}) = \binom{d}{l}$. 
\end{remark}

Using the previous lemma, we may relate a high $\Gsignperm$-alignment to the existence of a non-negligible Fourier coefficient of $O(1)$ or $d - O(1)$ degree.
\begin{lemma}\label{lem:bool-weak-char-connection}
Suppose $\{f_d\}_{d \in \NN}$ is such that $\mathcal{C}((f_d, \cH_d); \Gsignperm) \geq \Omega(d^{-C})$ for some constant $C > 0$. Then there is a constant $C' > 0$ such that for each $d$ there is $S_d \subseteq [d]$ with $|S_d| \leq C'$ or $|S_d| \geq d - C'$ and $|\hat{f}_d(S_d)| \geq \Omega(d^{-C'})$.
\end{lemma}
\begin{proof}
By Lemma~\ref{lem:technical-symm-bool-restated}, we have
\begin{align*}
\mathcal{C}_d((f_d, \cH_d); \Gsignperm) &\leq \sum_{\substack{S \subseteq [d] \\ \min(|S|,d-|S|) \leq C}} \hat{f}_d(S)^2 \binom{d}{|S|}^{-1}
+  \sum_{\substack{S \subseteq [d] \\ \min(|S|,d-|S|) > C}} \hat{f}_d(S)^2 \binom{d}{|S|}^{-1} \\
&\leq 
\sum_{\substack{S \subseteq [d] \\ \min(|S|,d-|S|) \leq C}} \hat{f}_d(S)^2
+  \sum_{\substack{S \subseteq [d] \\ \min(|S|,d-|S|) > C}} \hat{f}_d(S)^2 \binom{d}{C+1}^{-1} \\
&\leq \sum_{\substack{S \subseteq [d] \\ \min(|S|,d-|S|) \leq C}} \hat{f}_d(S)^2
+ O(d^{-C-1}) \\
&\leq \max_{\substack{S \subseteq [d] \\\min(|S|, d-|S|) \leq C}}  d^C \hat{f}_d(S)^2 + O(d^{-C-1}).
\end{align*}
So if $\mathcal{C}((f_d,\cH_d); \Gsignperm) \geq \Omega(d^{-C})$, then $$\max_{\substack{S \subseteq [d] \\\min(|S|, d-|S|) \leq C}} \hat{f}_d(S)^2 \geq \Omega(d^{-2C}).$$
Letting $C' = 2C$ proves the lemma.
\end{proof}

Now we are ready to prove the impossibility result of Theorem~\ref{thm:bool-weak-char}.
\begin{proof}[Proof of impossibility result of Theorem~\ref{thm:bool-weak-char}]
Suppose that $\{f_d, \cH_d\}$ is efficiently weak-learnable by GD-trained FC networks with symmetric initialization; let $\{\fNNd, \mu_{\btheta,d},\eta_d,k_d,R_d\tau_d\}$ be the sequence of networks, initializations, and GD hyperparameters satisfying Definition~\ref{def:fully-conn-weak-learn}, and in particular \eqref{eq:wl-condition} for some constant $C > 0$. By Proposition~\ref{prop:equi}, $(\fNNd, \mu_{\btheta,d})$-GD training is $G_{perm,sign}$-equivariant, so by Theorem~\ref{thm:gd-lower},
\begin{align*}
\PP[\|f_d - \fNNd(\cdot;\btheta_d)\|^2_{\Capitalltwo(\mu_d)} \leq \|f_d\|^2_{\Capitalltwo(\mu_d)} - d^{-C}] \leq \frac{\eta_d R_d \sqrt{k_d \mathcal{C}_d}}{2\tau_d} + \frac{\mathcal{C}_d}{d^{-C}} \leq O(d^{4c} \sqrt{\mathcal{C}_d} +  d^C \mathcal{C}_d),
\end{align*}
where $\mathcal{C}_d = \mathcal{C}((f_d, \cH_d); \Gsignperm)$, and we have used the bounds on the parameters guaranteed by Definition~\ref{def:fully-conn-weak-learn}.

On the other hand, since the network weak-learns \eqref{eq:wl-condition}, the right-hand side must be greater than or equal to $9/10$, so $\mathcal{C}_d \geq \Omega(d^{-\max(8c,C)})$.
By Lemma~\ref{lem:bool-weak-char-connection}, this implies there is a constant $C'$ and sequence $\{S_d\}$ such that $\min(|S_d|,d-|S_d|) \leq C'$ and $|\hat{f}_d(S_d)| \geq \Omega(d^{-C'})$.
\end{proof}

\subsubsection{Example: a concrete function that is not efficiently weak-learnable}\label{app:paritymod4}

As a concrete example of a bounded function on the Boolean hypercube that is not efficiently weak-learnable by \eqref{GD}, which was previously not known, consider the function $f \in L^2(\cH_d)$:
\begin{align*}
f(\bx) = \begin{cases} 0, & |\{j : x_j = 1\}| \equiv 0 \pmod{2} \\
1, & |\{j : x_j = 1\}| \equiv 1 \pmod{4} \\
-1, & |\{j : x_j = 1\}| \equiv 3 \pmod{4} \end{cases}.
\end{align*}
This function is invariant to permutations of the input, so its $\Gperm$-alignment is large. Therefore previous work that considers permutation symmetry of GD training of networks \cite{shalev2017failures,malach2020computational,abbe2020poly,abbe2022initial} cannot imply that this is a hard function to learn. Nevertheless, we will show that its $\Gsign$-alignment is small. First, let us rewrite $f$ as:
\begin{align*}
f(\bx) = \Im(\zeta(\bx)), \mbox{ where } \zeta(\bx) = \exp((\pi i / 2) \sum_{j=1}^{d} (x_j + 1)/2)
\end{align*}
Let $\bs,\bx \sim \cH_d$ and take the supremum over $h \in L^2(\cH_d)$ such that $\|h\|^2 = 1$ in the calculations below. By (a) permutation invariance of $f$, (b) using that $\Im(z)^2 \leq |z|^2$ for any complex $z \in \CC$, (c) a tensorization trick,
\allowdisplaybreaks
\begin{align*}
\mathcal{C}((f,\cH_d)&;\Gsignperm) \\
&\stackrel{(a)}{=} \mathcal{C}((f,\cH_d);\Gsign) \\
&= \sup_{h} \E_{\bs}[\E_{\bx}[f(\bx \odot \bs)h(\bx)]^2] \\
&\stackrel{(b)}{\leq} \sup_{h} \E_{\bs}[|\E_{\bx}[\zeta(\bx \odot \bs)h(\bx)]|^2] \\
&\stackrel{(c)}{\leq} \sup_{h} \E_{\bs}[\E_{\bx,\bx'}[\zeta(\bx \odot \bs)\bar{\zeta}(\bx' \odot \bs)h(\bx)h(\bx')]] \\
&\leq \sup_{h} \E_{\bx,\bx'}[h(\bx)h(\bx')\E_{\bs}[\zeta(\bx \odot \bs)\bar{\zeta}(\bx' \odot \bs)]] \\
&= \sup_{h} \E_{\bx,\bx'}[h(\bx)h(\bx')\E_{\bs}[\exp((\pi i / 2)(\sum_{j=1}^d (x_js_j + 1)/2 - \sum_{j=1}^d (x'_js_j + 1)/2)]] \\
&= \sup_{h} \E_{\bx,\bx'}[h(\bx)h(\bx') \prod_{j=1}^{d}\E_{s_j}[\exp((\pi i / 2)(s_j(x_j - x'_j)/2)]] \\
&= \sup_{h} \E_{\bx,\bx'}[h(\bx)h(\bx')\prod_{j=1}^d(\frac{1}{2}\exp((\pi i / 2)((x_j - x'_j)/2)) + \frac{1}{2}\exp((\pi i / 2)((-x_j + x'_j)/2)))] \\
&= \sup_{h} \E_{\bx,\bx'}[h(\bx)h(\bx')\prod_{j=1}^d\delta_{x_jx'_j}] \\
&= \sup_{h} \E_{\bx,\bx'}[h(\bx)h(\bx') \delta_{\bx,\bx'}] \\
&= \sup_{h} \E_{\bx}[h(\bx)^2] / 2^d \\
&= \frac{1}{2^d}.
\end{align*}

So, $\mathcal{C}((f,\cH_d);\Gsignperm) \leq 1/2^d$, which is negligible. Plugging this into Theorem~\ref{thm:gd-lower} implies that $f$ cannot be efficiently learned by FC networks with sign-flip symmetric initialization trained by \eqref{GD}.

\subsubsection{Proof of achievability result of Theorem~\ref{thm:bool-weak-char}}

We prove the converse achievability result with the following training setup. Consider a two-layer FC network architecture with parameters $\btheta = (\bW, \ba)$ where $\bW \in \R^{d \times m}$ and $\ba \in \R^m$ and $\fNN(\bx;\btheta) = \ba^{\top} \sigma(\bW \bx)$, where the activation function $\sigma$ is a bump function applied elementwise: \begin{align*}
\sigma(z) = \begin{cases} 1, & -0.5 \leq z \leq 1.5 \\
0, & z \not\in (-0.5, 1.5)
\end{cases}.
\end{align*}
Set the initialization distribution $\mu_{\btheta} = \mu_{\bW} \times \mu_{\ba}$, where $\mu_{\ba} = \delta_{\bzero}$, and $\bW \sim \mu_w^{\otimes (m \times d)}$, where $$\mu_w = (1-p) \delta_0 + \frac{p}{2} \delta_1 + \frac{p}{2} \delta_{-1},$$ for some parameter $p \in [0,1]$. Note that the above is a FC architecture with an i.i.d. sign-flip symmetric initialization.\footnote{To fully ensure i.i.d. initialization for both layers, we could take $\mu_{\ba} = \mu_w^{\otimes m}$. However it is notationally simpler to take $\mu_{\ba} = \delta_{\bzero}$, and furthermore any bounded initialization works for $\mu_{\ba}$, because the same random features analysis goes through.} We obtain the following guarantee for training the network with GD.

\begin{lemma}\label{lem:gd-boolean-achievability}
Let $f : \cH_d \to \R$, and let $\cD \in \cP(\cH_d \times \R)$ be the distribution of $(\bx, f(\bx))$ where $\bx \sim \cH_d$. Let $s \in \{0,\ldots,d\}$. Define $r_s = \max_{S, |S| = s} |\hat{f}(S)| / (d^2  \binom{d}{s})$.

Consider training the above architecture and initialization with parameter $p = s/d$, width $m \geq 100/(r_s)^4$. If we run \eqref{GD} for one step, with learning rate $0 < \eta < 1/(4m)$ and noise level $\tau \leq \eta (r_s)^4 / (100 m^2 d)$, then with probability at least 9/10 we have
$$\ell_{\cD}(\btheta^1) \leq \|f\|^2_{\Capitalltwo(\cH_d)} -  \eta (r_s)^4 / 4.$$
\end{lemma}
\begin{proof}
Let us prove that the loss is smooth in a small ball around initialization: For any $\bx \in \cH_d$ and $\tilde{\bW}$ such that $\|\tilde{\bW} - \bW^0\|_{\infty} \leq 1/(4d)$, we have $\tilde{\bW} \bx \in \ZZ + [-1/4, 1/4]$ almost surely. Therefore $\sigma'(\tilde{\bW} \bx) = \sigma''(\tilde{\bW} \bx) = \bzero$. So for any $\tilde{\ba}$, one can show that the loss is smooth at $\tilde{\btheta} = (\tilde{\bW}, \tilde{\ba})$:
\begin{align}\label{eq:gd-boolean-achievability-smoothness}
\nabla^2_{\btheta} \ell_{\cD}(\tilde{\btheta}) \lesssim m \|\sigma\|_{\infty}^2 \bI \lesssim m \bI.
\end{align}
Let us lower-bound the magnitude of the gradient. Since $\sigma'(\bW^0 \bx) = 0$, we have $$\|\nabla_{\bW} \ell_{\cD}(\btheta^0)\|_F^2 = \|\bzero\|_F^2 = 0.$$
However, the gradient with respect to $\ba$ is nonzero.  Writing $\bW^0 = [\bw_1^0,\ldots,\bw_m^0]$, since $\ba^0 = \bzero$,
 \begin{align}\label{eq:boolean-achievability-grad-expected}\|\nabla_{\ba} \ell_{\cD}(\btheta^0)\|^2 &= \|\E_{\bx}[(f(\bx) - (\ba^0)^{\top} \sigma(\bW^0 \bx)) \sigma(\bW^0 \bx)]\|^2 = \sum_{i \in [m]} \E_{\bx}[f(\bx) \sigma(\<\bw_i^0,\bx\>)]^2.
 \end{align}
 Our main claim is that each term in this sum is lower-bounded in expectation.
\begin{claim}\label{claim:gd-boolean-achievability-technical}
For any $i \in [m]$, $\E_{\bw_i^0}[\E_{\bx}[f(\bx) \sigma(\<\bw_i^0,\bx\>)]^2] \geq \frac{1}{d^4} \binom{d}{s}^{-2} \sum_{S, |S| = s} |\hat{f}(S)|^2 \geq (r_s)^2$.
\end{claim}
\begin{proof}[Proof of claim]
For ease of notation, let $\bw \sim \cL(\bw_i^0)$. Write $\bw = \ba \odot \bone_{S}$ for $\ba \sim \cH_d$ and $S \subset [d]$ where each element of $[d]$ is in $S$ independently with probability $p$. Then
\begin{align}
\E_{\bw}&[\E_{\bx}[f(\bx) \sigma(\<\bw,\bx\>)]^2] \\
&= \E_{S,\ba}[\E_{\bx}[(-1)^{\floor{|S|/2}}\sgn(\hat{f}(S)) \chi_{S} (\ba) f(\bx) \sigma(\<\ba \odot \bone_S,\bx\>)]^2] \nonumber \\
&\geq \E_{S,\ba,\bx}[(-1)^{\floor{|S|/2}}\sgn(\hat{f}(S)) \chi_{S} (\ba) f(\bx) \sigma(\<\ba \odot \bone_S,\bx\>)]^2 \nonumber \\
 &= (\sum_{S' \subseteq [d]} \E_{S}[(-1)^{\floor{|S|/2}}\sgn(\hat{f}(S))\hat{f}(S') \E_{\ba,\bx}[\chi_S(\ba) \chi_{S'}(\bx) \sigma(\<\ba \odot \bone_S, \bx\>)]])^2 \nonumber \\
 &= (\sum_{S' \subseteq [d]} \E_S[(-1)^{\floor{|S|/2}}\sgn(\hat{f}(S)) \hat{f}(S') c_{S,S'}])^2, \label{eq:cSSuse}
\end{align}
where in the last line we define $c_{S,S'} = \E_{\ba,\bx}[\chi_S(\ba) \chi_{S'}(\bx) \sigma(\<\ba \odot \bone_S, \bx\>)]$. If $S \not\subseteq S'$, then we have $c_{S,S'} = 0$, since $\E_{\ba}[\chi_S(\ba) \sigma(\<\ba \odot \bone_S,\bx\>)]$ depends only on $\bx_S$ and not on $\bx_{[d] \sm S}$. Furthermore, for any $S' \subseteq S \subseteq [d]$,
\begin{align*}
c_{S,S'} &= \E_{\ba,\bx}[\chi_{S}(\ba)\chi_{S'}(\bx) \delta(\<\ba \odot \bone_S, \bx\> \in \{0,1\})] \\
&= \E_{\ba,\bx}[\chi_{S}(\ba)\chi_{S'}(\bx) \sum_{k=-d}^d \delta(\<\ba \odot \bone_S, \bx\> \in \{0,1\})\delta(\<\ba \odot \bone_{S'}, \bx\> = k)] \\
&= \E_{\ba,\bx}[\chi_{S}(\ba)\chi_{S'}(\bx) \sum_{k=-d}^d \delta(\<\ba \odot \bone_{S \sm S'}, \bx\> \in \{-k,-k+1\})\delta(\<\ba \odot \bone_{S'}, \bx\> = k)].
\end{align*}
Notice that $\<\ba \odot \bone_{S'}, \bx\> = \sum_{i \in S'} a_i x_i = |S'| - 2 |\{i \in S' : a_ix_i = -1\}|$. Also, $\chi_{S'}(\ba) \chi_{S'}(\bx) = (-1)^{|\{i \in S' : a_ix_i = -1\}|}$. Therefore, $\chi_{S'}(\ba) \chi_{S'}(\bx) = (-1)^{(|S'| - \<\ba \odot \bone_{S'}, \bx\>)/2}$, so plugging this into the above,
\begin{align*}
c_{S,S'} &= \E_{\ba,\bx}[\chi_{S \sm S'}(\ba) \sum_{k=-d}^d (-1)^{(|S'| - k)/2} \delta(\<\ba \odot \bone_{S \sm S'}, \bx\> \in \{-k,-k+1\})\delta(\<\ba \odot \bone_{S'}, \bx\> = k)] \\
&= \E_{\ba}[\chi_{S \sm S'}(\ba) \sum_{k=-d}^d (-1)^{(|S'| - k)/2} \PP_{\bx_{S \sm S'}}[\<\ba \odot \bone_{S \sm S'}, \bx\> \in \{-k,-k+1\}]\PP_{\bx_{S'}}[\<\ba \odot \bone_{S'}, \bx\> = k]] \\
&= \E_{\ba}[\chi_{S \sm S'}(\ba)] \sum_{k=-d}^d (-1)^{(|S'| - k)/2} \PP_{\bx_{S \sm S'}}[\<\bone_{S \sm S'}, \bx\> \in \{-k,-k+1\}]\PP_{\bx_{S'}}[\<\bone_{S'}, \bx\> = k],
\end{align*}
where the last line uses the sign-flip symmetry of the distributions $\ba,\bx \sim \cH_d$.
Therefore if $S \neq S'$, we have $c_{S,S'} = 0$. Otherwise, if $S = S'$, we have
\begin{align*}
c_{S,S} = \sum_{k \in \{0,1\}} (-1)^{(|S| - k)/2} \PP_{\bx}[\<\bone_{S}, \bx\> = k] = (-1)^{\floor{|S|/2}} \PP_{\bx}[\<\bone_{S}, \bx\> \in \{0,1\}],
\end{align*}
where we have used that $\<\bone_{S},\bx\>$ has the same parity as $|S|$ almost surely over $\bx$, so either $\PP_{\bx}[\<\bone_{S}, \bx\> = 0] = 0$ or $\PP_{\bx}[\<\bone_S, \bx\> = 1] = 0$. Plugging this into \eqref{eq:cSSuse},
\begin{align*}
\E_{\bw}[\E_{\bx}[f(\bx)\sigma(\<\bw,\bx\>)]^2] &= \E_{S}[|\hat{f}(S)| \PP_{\bx}[\<\bone_{S}, \bx\> \in \{0,1\}]]^2 \\
&= \E_S[|\hat{f}(S)| \binom{|S|}{\floor{|S|/2}} / 2^{|S|}]^2 \\
&\geq \left(\frac{1}{d} \sum_{S, |S| = s} \binom{d}{s}^{-1} |\hat{f}(S)| \binom{|S|}{\floor{|S|/2}} / 2^{|S|}\right)^2 \\
&\geq \frac{1}{d^4} \binom{d}{s}^{-2} \sum_{S, |S| = s} |\hat{f}(S)|^2 \\
&\geq (r_s)^2, 
\end{align*}
where recall that $s = pd$ and we use that with probability at least $1/d$ the random set $S$ where each element is included with probability $p$ has size $|S| = s$.
\end{proof}
Combining the above claim with a Hoeffding bound, we can bound the gradient with high probability.

\begin{claim}
Let $E_1$ denote the event that 
$(r_s)^2 / 2 \leq \|\nabla_{\ba} \ell_{\cD}(\btheta^0)\|^2 /m \leq 1$. Then $\PP_{\btheta^0}[E_1] \geq 19/20$.
\end{claim}
\begin{proof}[Proof of claim]
In \eqref{eq:boolean-achievability-grad-expected}, $\|\nabla_{\ba} \ell_{\cD}(\btheta^0)\|^2$ is written as a sum of $m$ i.i.d. random variables of the form $\E_{\bx}[f(\bx) \sigma(\<\bw_i^0,\bx\>)]^2$. These are almost surely in the range $[0,1]$ since $\|f(\bx)\|_{\infty} \leq 1$ and $\|\sigma\|_{\infty} \leq 1$. This proves the upper bound almost surely.

For the lower bound, by Hoeffding's inequality,
$$\PP[\|\nabla_{\ba} \ell_{\cD}(\btheta^0)\|^2 /m \leq \E[\|\nabla_{\ba} \ell_{\cD}(\btheta^0)\|^2 /m] - 5 / \sqrt{m}] \leq \exp(-2 \cdot 25) < 1/20.$$
Finally, recall that $$\E[\|\nabla_{\ba} \ell_{\cD}(\btheta^0)\|^2 /m] = \E_{\bw_i}[\E_{\bx}[f(\bx) \sigma(\<\bw_i^0,\bx\>)]^2] \geq (r_s)^2$$ by Claim~\ref{claim:gd-boolean-achievability-technical}. Since we have taken $m \geq 100/(r_s)^4$, we have
$(r_s)^2 - 5 / \sqrt{m} \geq (r_s)^2 / 2$, which concludes the claim.
\end{proof}
We also bound the magnitude of the added noise.
\begin{claim}
Let $E_2$ denote the event that $\|\bxi^0\|_{\infty} \leq \min(1/(4d), \frac{\eta (r_s)^4}{4m^2d}) $. Then $\PP[E_2] \geq 19/20$.
\end{claim}
\begin{proof}[Proof of claim]
By a union bound over the Gaussian tail bounds for all $md + m$ entries of $\bxi^0$.
\end{proof}
Condition on events $E_1$ and $E_2$, which hold with probability at least $9/10$ by the above two claims. Since, $\bW^1 = \bW^0 + \bxi^0$, we have $\|\bW^1 - \bW^0\|_{\infty} \leq 1/(4d)$. By (a) the smoothness in \eqref{eq:gd-boolean-achievability-smoothness}, and (b) the learning rate choice $\eta \leq 1/(4m)$, (c) by the above bounds,
\begin{align*}
\ell_{\cD}(\btheta^1) &\stackrel{(a)}{\leq} \ell_{\cD}(\btheta^0) - \eta \|\nabla_{\btheta} \ell_{\cD}(\btheta^0)\|^2 + |\<\nabla_{\btheta}\ell_{\cD}(\btheta^0), \bxi^0\>| + m \|\eta^2 \nabla_{\btheta}\ell_{\cD}(\btheta^0) + \bxi^0\|^2 \\
&\leq \|f\|^2_{\Capitalltwo(\cH_d)} - \eta \|\nabla_{\btheta} \ell_{\cD}(\btheta^0)\|^2(1 - 2m\eta) + \sqrt{m}\sqrt{(md+m)}\|\bxi^0\|_{\infty} + 2 m (md+m)\|\bxi^0\|_{\infty}^2 \\
&\stackrel{(b)}{\leq} \|f\|^2_{\Capitalltwo(\cH_d)} - \frac{\eta}{2} \|\nabla_{\btheta} \ell_{\cD}(\btheta^0)\|^2 + \sqrt{m}\sqrt{(md+m)}\|\bxi^0\|_{\infty} + 2 m (md+m)\|\bxi^0\|_{\infty}^2 \\
&\leq \|f\|^2_{\Capitalltwo(\cH_d)} - \frac{\eta}{4} (r_s)^4,
\end{align*}
which proves the lemma.
\end{proof}

Lemma~\ref{lem:gd-boolean-achievability} implies the achievability result of Theorem~\ref{thm:bool-weak-char}:
\begin{proof}[Proof of achievability result of Theorem~\ref{thm:bool-weak-char}]
Suppose that $\{f_d\}$ is a sequence of functions $f_d \in \Capitalltwo(\cH_d)$ such that $\|f_d\|_{\Capitalltwo(\cH_d)} \leq 1$ and there is a constant $C > 0$ such that $$\max_{\substack{S \subseteq [d] \\ \min(|S|,d-|S|) \leq C}} \hat{f}_d(S)^2 \geq \Omega(d^{-C}).$$ Then for each $d \in \NN$ there is $s_d \in \{0,\ldots,d\}$ such that $\max_{S \subseteq [d], |S| = s_d} |\hat{f}(S)| / (d^2 \binom{d}{s_d}) \geq \Omega(d^{-2C-2})$. So by Lemma~\ref{lem:gd-boolean-achievability} there is a FC architecture with width $m \leq O(d^{8C+8})$ such that initializing i.i.d. from a symmetric distribution and GD-training with learning rate $\eta \leq 1$ and noise level $\tau \geq \Omega(d^{-40C-41})$ reaches a loss $\leq \|f_d\|_{\Capitalltwo(\cH_d)}^2 - \Omega(d^{-12C-12})$ with probability at least $9/10$. One can take $R = \poly(d)$, since under events $E_1$ and $E_2$ in the analysis of Lemma~\ref{lem:gd-boolean-achievability} we only optimize in parameter regions with $\poly(d)$-size gradients.
\end{proof}

\subsection{Functions on unit sphere: proof of Theorem~\ref{thm:sphere-weak-char}}\label{app:sphere-char}

We now prove the analogous result for functions on the unit sphere, again proving the impossibility and achievability results separately. We use that we can decompose $\Capitalltwo(\S^{d-1})$ as
\begin{align*}
    \Capitalltwo(\S^{d-1}) = \oplus_{l=0}^{\infty} \cV_{d,l},
\end{align*}
where $\cV_{d,l}$ is the space of degree-$l$ spherical  harmonics, which has dimension $\dim(\cV_{d,l}) = \frac{2l+d-2}{l} \binom{l+d-3}{l-1}$ (see, e.g., \cite{hochstadt2012functions}). Let $\Pi_{\cV_{d,l}} : \Capitalltwo(\S^{d-1}) \to \cV_{d,l}$ denote the projection onto the degree-$l$ spherical harmonics.

\subsubsection{Proof of impossibility result of Theorem~\ref{thm:sphere-weak-char}} Again, we apply Theorem~\ref{thm:gd-lower} to prove the impossibility result. Since the FC networks are i.i.d. Gaussian-initialized, GD training is $\Grot$-equivariant by Proposition~\ref{prop:equi}. So recall the bound on the $\Grot$-alignment proved in the main text.

\begin{lemma}[Restatement of Lemma~\ref{lem:rotation-sphere-alignment-old}]\label{lem:rotation-sphere-alignment}
Let $f \in \Capitalltwo(\S^{d-1})$. Then $$\mathcal{C}((f, \S^{d-1}); \Grot) = \max_{l \in \ZZ_{\geq 0}} \|\Pi_{\cV_{d,l}} f\|^2 / \dim(\cV_{d,l}).$$
\end{lemma}

We may now conclude the proof of impossibility of learning.

\begin{proof}[Proof of impossibility result of Theorem~\ref{thm:sphere-weak-char}]
Suppose $\{f_d, \S^{d-1}\}$ is efficiently weak-learnable by GD-trained FC networks with Gaussian initialization; let $\{\fNNd, \mu_{\btheta,d},\eta_d,k_d,R_d,\tau_d\}$ be the sequence of networks, initializations, and GD hyperparameters achieving \eqref{eq:wl-condition} for some constant $C$. On the other hand, by Proposition~\ref{prop:equi}, GD training is $\Grot$-equivariant, so by Theorem~\ref{thm:gd-lower}
\begin{align*}
\PP[\|f_d - \tilde{f}_d\|^2_{\Capitalltwo(\S^{d-1})} \leq \|f_d\|^2_{\Capitalltwo(\S^{d-1})} - d^{-C}] \leq \frac{\eta_d R_d \sqrt{k_d \mathcal{C}_d}}{2\tau_d} + \frac{\mathcal{C}_d}{d^{-C}} \leq O(d^{4c} \sqrt{\mathcal{C}_d} +  d^C \mathcal{C}_d),
\end{align*}
where $\mathcal{C}_d = \mathcal{C}(f_d; \S^{d-1}, \Grot)$. By \eqref{eq:wl-condition}, the right-hand side must be greater than $9/10$, so $\mathcal{C}_d \geq \Omega(d^{-\max(8c,C)})$.
By Lemma~\ref{lem:rotation-sphere-alignment}, this implies that for each $d \in \NN$, there is $l_d \in \ZZ_{\geq 0}$ such that $\|\Pi_{\cV_{d,l_d}} f_d\|^2 / \dim(\cV_{d,l_d}) \geq \Omega(d^{-\max(8c,C)})$. Since $\dim(\cV_{d,l_d}) \geq 1$, this means $\|\Pi_{\cV_{d,l_d}} f_d\|^2 \geq \Omega(d^{-\max(8c,C)})$. Now let us prove that $l_d$ is upper-bounded by a constant. Since $\|\Pi_{\cV_{d,l_d}} f_d\|^2 \leq \|f_d\|^2 \leq 1$, we must have $\dim(\cV_{d,l_d}) \leq O(d^{\max(8c,C)})$. Since $\dim(\cV_{d,l}) \geq (d/l) \cdot ((l + d - 3) / (l-1))^{l-1}$, it follows that $l_d \leq \min(8c, C)$ for all sufficiently large $d$.
\end{proof}

\subsubsection{Proof of achievability result of Theorem~\ref{thm:sphere-weak-char}}
The positive weak learning result can again be achieved by two-layer FC networks. Indeed, in Theorem~1(b) of \cite{ghorbani2021linearized} it is proved that for any $l \in \ZZ_{\geq 1}$, a two-layer neural network random features model with $d^{\omega(l)}$ neurons will learn the projection of the target function to the degree-$l$ spherical harmonics. The 2-layer model is given by $\fNN(\bx;\btheta) = \ba^{\top} \sigma(\bW \bx)$, where only $\ba \in \R^m$ is trained, and $\bW = [\bw_1,\ldots,\bw_m] \in \R^{m \times d}$ is fixed to its initialization. Unfortunately, the first-layer initialization used in \cite{ghorbani2021linearized} is not Gaussian initialization but rather uniform over the unit sphere $\bw_i^0 \sim \S^{d-1}$. Therefore, we must modify their analysis, and for simplicity we analyze only one step of GD training on the second-layer weights, when we have $\bw_i^0 \stackrel{i.i.d.}{\sim} \cN(0,I_d/\sqrt{d})$, and $\ba^0 = \bzero$.
\begin{lemma}\label{lem:gd-sphere-achievability}
Let $f \in \Capitalltwo(\S^{d-1})$, and let $\cD \in \cP(\S^{d-1} \times \R)$ be the distribution of $(\bx, f(\bx))$, where $\bx \sim \S^{d-1}$.

Let $l \in \ZZ_{\geq 0}$ and $c > 0$. Suppose that $\|\Pi_{\cV_{d,l}} f\|^2 \geq d^{-c}$. Consider training the above architecture and initialization with activation function $\sigma(t) = t^l$, width $m \geq d^{4l+4c+1}$. If we run \eqref{GD} for one step, with learning rate $0 < \eta < d^{-8l+8c-2}/m$ and noise level $\tau \leq \eta d^{-8l+8c-3} / (100m^2 d)$, then with probability at least $9/10$ for large enough $d$ we have
\begin{align*}
\ell_{\cD}(\btheta^1) \leq \|f\|^2_{\Capitalltwo(\S^{d-1})} - \eta d^{-8l+8c+2}
\end{align*}
\end{lemma}
\begin{proof}
Similarly to the proof of Lemma~\ref{lem:gd-boolean-achievability}, the proof relies on upper-bounding the smoothness of $\fNN$ in a ball around initialization, as well as lower-bounding the gradient at initialization.

First, lower-bound the gradient:
\begin{claim}
$\PP[\|\nabla_{\ba} \ell_{\cD}(\btheta^0)\|^2 \geq md^{-2l-2c-1/2}] \geq 49/50$.
\end{claim}
\begin{proof}
\begin{align*}
\nabla_{\ba} \ell_{\cD}(\btheta^0) &= \E_{(\bx,y) \sim \cD}[-y \nabla_{\ba} \fNN(\bx; \btheta^0)] = \E_{\bx \sim \S^{d-1}}[-f(\bx) \sigma(\bW^0 \bx)]
\end{align*}
Let $C > 0$ be a constant. For any $i \in [m]$, let $E_i$ be the event that $\|\bw_i^0\| \in [1-d^{-C}, 1+d^{-C}]$. Note $\PP[E_i] \geq \Omega(d^{-C})$.
And conditioned on the event $E_i$, we have
\begin{align*}|\EE_{\bx \sim \S^{d-1}}[f(\bx) \sigma(\<\bw_i^0, \bx\>) \mid E_i] - \EE_{\bx \sim \S^{d-1}}[f(\bx) \sigma(\<\frac{\bw_i^0}{\|\bw_i^0\|}, \bx\>) \mid E_i]| \leq \|f\| l2^{l-1} d^{-C} = O(d^{-C}). \end{align*}

So letting $\bv_i^0 = \bw_i^0 / \|\bw_i^0\|$ we have 
\begin{align*}
\EE_{\bx \sim \S^{d-1}}[f(\bx) \sigma(\<\bw_i^0, \bx\>) \mid E_i]^2 \geq \EE_{\bx \sim \S^{d-1}}[f(\bx) \sigma(\<\bv_i^0, \bx\>)]^2 - O(d^{-C})
\end{align*}
Finally, by (a) a tensorization trick, (b) the rotational invariance of the distribution $\bv_i^0$, (c) the Schur orthogonality relations for $\Grot$, 
\begin{align*}
\E_{\bv_i^0}[&\E_{\bx \sim \S^{d-1}}[f(\bx)\sigma(\<\bv_i^0, \bx\>)]^2] \\
&\stackrel{(a)}{=} \E_{\bv_i^0}[\E_{\bx,\bx' \sim \S^{d-1}}[f(\bx)f(\bx')\sigma(\<\bv_i^0, \bx\>)\sigma(\<\bv_i^0, \bx'\>)]] \\
&\stackrel{(b)}{=} \E_{g \sim \Grot, \bv_i^0}[\E_{\bx,\bx' \sim \S^{d-1}}[f(\bx)f(\bx')\sigma(\<g(\bv_i^0), \bx\>)\sigma(\<g(\bv_i^0), \bx'\>)]] \\
&= \E_{g \sim \Grot, \bv_i^0}[\E_{\bx,\bx' \sim \S^{d-1}}[f(\bx)f(\bx')\sigma(\<\bv_i^0, g^{-1}(\bx)\>)\sigma(\<\bv_i^0, g^{-1}(\bx')\>)]] \\
&= \E_{\bv_i^0}[\E_{g \sim \Grot}[\<f, \sigma(\<\bv_i^0, g^{-1}(\cdot)\>)\>_{\Capitalltwo(\S^{d-1})} \< f, \sigma(\<\bv_i^0, g^{-1}(\cdot)\>)\>_{\Capitalltwo(\S^{d-1})}]] \\
&\stackrel{(c)}{=} \E_{\bv_i^0}[\sum_{r=0}^{\infty} \frac{1}{\dim(\cV_{d,r})} \|\Pi_{\cV_{d,r}} f\|^2_{\Capitalltwo(\S^{d-1})} \|\Pi_{\cV_{d,r}} \sigma(\<\bv_i^0, \cdot\>)\|^2_{\Capitalltwo(\S^{d-1})}] \\
&\geq \Omega(d^{-2l} \cdot \|\Pi_{\cV_{d,l}} f\|^2)
\end{align*}
So combined we have
\begin{align*}
\E_{\bv_i^0}[\EE_{\bx \sim \S^{d-1}}[f(\bx) \sigma(\<\bw_i^0, \bx\>) \mid E_i]^2] \geq \Omega(d^{-C-2l} \cdot \|\Pi_{\cV_{d,l}} f\|^2) - O(d^{-2C}).
\end{align*}
Taking $C = 2l+2c$, we get 
\begin{align*}\E_{\bv_i^0}[\EE_{\bx \sim \S^{d-1}}[f(\bx) \sigma(\<\bw_i^0, \bx\>) \mid E_i]^2] \geq d^{-C}.\end{align*}
So since $m \geq 2C+1$, by Hoeffding's inequality we get
\begin{align*}\PP_{\bW^0}[\|\nabla_{\ba} \ell_{\cD}(\btheta^0)\|^2 \leq m d^{-C-1/2}] \ll 1,\end{align*}
proving the claim.
\end{proof}

The rest of the proof proceeds similarly to the achievability proof of Theorem~\ref{thm:bool-weak-char}. We can upper-bound the smoothness under the event that $\bW^0$ lies in a ball of radius 2, which occurs with high probability. And we can upper-bound the magnitude of the noise by Gaussian tail bounds. Therefore, the $1/\poly(d)$ lower bound on the gradient implies that a $\geq 1/\poly(d)$ amount of progress is made towards learning when training $\ba$ for one step. Finally, we can again take $R = \poly(d)$ since we can ensure that the gradients remain $\poly(d)$-bounded.

\end{proof}

\begin{proof}[Proof of achievability result of Theorem~\ref{thm:sphere-weak-char}]
The proof is analogous to the achievability result of Theorem~\ref{thm:bool-weak-char}, but using Lemma~\ref{lem:gd-sphere-achievability} instead of Lemma~\ref{lem:gd-boolean-achievability}.
\end{proof}

\section{Extension of MSP necessity result, proof of Theorem~\ref{thm:msp-bool}}\label{app:msp}
\begin{proof}[Proof of Theorem~\ref{thm:msp-bool}]
For any set $T \subseteq [P]$, let $\cS^{\not\subseteq T} = \{S \subseteq [P] : \hat{f}_*(S) \neq 0, S \not\subseteq T\}$. Since $h_*$ is not $\ell$-MSP, we can choose $T$ such that $\cS^{\not\subseteq T} \neq \emptyset$, and for all $S \in \cS^{\not\subseteq T}$, we have $|S \sm T| \geq \ell+1$.

Let $G' \subseteq \Gperm$ be the subgroup of permutations on indices $[d] \sm T$. Furthermore, let $\alpha(\bx) = \sum_{S \subseteq T} \hat{f}_*(S) \chi_S(\bx)$. Notice that $\alpha$ is invariant to the action of $G'$, since each permutation $\sigma \in G'$ keeps the indices $T$ fixed. Furthermore, $f_* - \alpha = \sum_{S \in \cS^{\not\subseteq T}} \hat{f}_*(S) \chi_S(\bx)$, so $$\|f_* - \alpha\|^2 > c_{h_*},$$ where $c_{h_*} = \min_{S \subseteq [P], \hat{h}_*(S) \neq 0} |\hat{h}_*(S)|^2 > 0$ is a constant depending only on $h_*$.

Let us now bound the $G'$-alignment of $f_* - \alpha$. For any $\beta \in \Capitalltwo(\cH_d)$, by Cauchy-Schwarz
\begin{align*}
\E_{\sigma \sim G'}[\<f_* \circ \sigma - \alpha \circ \sigma, \beta\>^2_{\Capitalltwo(\cH_d)}] &= \E_{\sigma \sim G'}[\<\sum_{S \in \cS^{\not\subseteq T}} \hat{f}_*(S) \chi_S \circ \sigma, \beta\>^2_{\Capitalltwo(\cH_d)}] \\
&\leq 2^P \sum_{S \in \cS^{\not\subseteq T}} \hat{f}_*(S)^2 \E_{\sigma \sim G'}[\<\chi_S \circ \sigma, \beta\>^2_{\Capitalltwo(\cH_d)}] \\
\end{align*}
And for any $S \in \cS^{\not\subseteq T}$, define $G'(S) = \{\sigma(S) : \sigma \in G'\}$. Then (a) by Parseval's, using the orthogonality of the Fourier coefficients, and (b) since $|S \sm T| \geq \ell+1$ by construction for all $S \in \cS^{\not\subseteq T}$, we have $|G'(S)| = (d - |T|)! / (d - |T| - |S \sm T|)! \geq d^{\ell+1} / 2^P$, for any $d \geq 2P$,
\begin{align*}
\E_{\sigma \sim G'}[\<\chi_S \circ \sigma, \beta\>^2_{\Capitalltwo(\cH_d)}] &= \frac{1}{|G'(S)|} \sum_{S' \in G'(S)} \<\chi_{S'}, \beta\>_{\Capitalltwo(\cH_d)}^2 \stackrel{(a)}{\leq} \frac{1}{|G'(S)|} \|\beta\|^2_{\Capitalltwo(\cH_d)} &\stackrel{(b)}{\leq} \|\beta\|^2_{\Capitalltwo(\cH_d)} 2^P / d^{\ell + 1}.
\end{align*}
We conclude that 
\begin{align*}
\E_{\sigma \sim G'}[\<f_* \circ \sigma - \alpha \circ \sigma, \beta\>^2_{\Capitalltwo(\cH_d)}] &\leq \|\beta\|^2_{\Capitalltwo(\cH_d)} 2^{2P}\sum_{S \in \cS^{\not\subseteq T}} \hat{f}_*(S)^2 / d^{\ell + 1} \leq \|\beta\|^2_{\Capitalltwo(\cH_d)} C_{h_*} / d^{\ell + 1},
\end{align*}
where $C_{h_*}$ is a constant depending only on $h_*$. So $$\mathcal{C}(f_* - \alpha; \cH_d, G') \leq C_{h_*} / d^{\ell + 1}.$$ Set $\eps_0 = c_{h_*} / 2$. By Theorem~\ref{thm:gd-lower} have
\begin{align*}
\PP_{\btheta^k}[\ell_{\cD}(\btheta^k) \leq \eps_0] \leq \frac{\eta L}{2 \tau}\sqrt{\frac{C_{h_*} k}{d^{\ell+1}}} + \frac{2C_{h_*}}{d^{\ell+1} c_{h_*}} \leq \frac{C_{h_*}\eta L}{\tau} \sqrt{\frac{k}{d^{\ell+1}}} + \frac{2C_{h_*}}{d^{\ell+1}}.
\end{align*}
\end{proof}

\section{Hardness results for SGD}

In this section, for $\gamma > 0$, we let $\cD(f,\mu_{\cX},\gamma) \in \cP(\cX \times \R)$ denote the distribution of $(\bx, f(\bx) + \xi)$ where $\bx \sim \mu_{\cX}$ and $\xi \sim \cN(0,\gamma^2)$ is independent noise.

\subsection{Warm-up: hardness from permutation equivariance}

Before proving Theorem~\ref{thm:sgd-sum-mod-8-hard}, we warm up with a simple observation that exploits permutation equivariance.

\begin{theorem}\label{thm:sgd-hard-warm-up}
For any $d$, define $f_d : \cH_d \to \{+1,-1\}$ to be the parity of the first $\floor{d/2}$ bits, i.e., $$f_d(\bx) = \prod_{i=1}^{\floor{d/2}} x_i.$$ Let $\{\fNNd, \mu_{\btheta,d}\}_{d \in \N}$ be a family of networks and initializations satisfying Assumption~\ref{ass:noskip} (FC, i.i.d. initialization). Let $\gamma > 0$, and let $\{n_d\}$ be sample sizes such that $(\fNNd, \mu_{\btheta_d})$-SGD training on $n_d$ samples from $\cD(f_d, \cH_d, \gamma)$ rounded to $\poly(d)$ bits yields parameters $\btheta_d$ with
\begin{align}\label{eq:sgd-perm-hard-pre-markov}
\E_{\btheta_d}[\|f_d - \fNN(\cdot; \btheta_d)\|^2] \leq 0.01.
\end{align}
Then, under $\gamma$-LPGN hardness, it is not possible to run $(\fNNd, \mu_{\btheta,d})$-SGD on $n_d$ samples and evaluate $\fNN(\cdot;\btheta_d)$ in $\poly(d)$ time.
\end{theorem}
This result follows straightforwardly from $G_{perm}$-equivariance of SGD in Proposition~\ref{prop:equi}, which states that SGD on FC networks with i.i.d. initialization is blind to the order of the input vectors' coordinates, up to an unknown shared permutation. So if $(\fNNd,\mu_{\btheta,d})$-SGD can learn $f(\bx) = \prod_{i=1}^{\floor{d/2}} x_i$ from $\gamma$-noisy samples, it can learn $f(\bx) = \chi_S(\bx)$ from noisy samples, for any $S \subseteq [d]$ with $|S| = \floor{d/2}$. But that is a $\gamma$-LPGN-hard problem.
\begin{proof}[Proof of Theorem~\ref{thm:sgd-hard-warm-up}]
Formally, consider the algorithm $\ASGD_d$ that runs $(\fNNd, \mu_{\btheta,d})$-SGD on $n_d$ samples $(\bx_i,y_i)_{i \in [n_d]}$ and outputs $\ASGD_d((\bx_i,y_i)_{i \in [n_d]}) : \cH_d \to \R$ given by $[\ASGD_d((\bx_i,y_i)_{i \in [n_d]})](\cdot) = \fNNd(\cdot;\btheta_d)$.
Let us prove that $\ASGD_d$ achieves probability of error $\leq 0.01$ on the $\gamma$-LPGN problem.

Let $(S, \bq, (x_i,y_i)_{i \in [n_d]})$ be a $(d,n_d,\gamma)$-LPGN instance, where $S \subseteq [d]$ is unknown, $\bq \sim \cH_d$ is random, and $(\bx_i,y_i)_{i \in [n_d]}$ is random with $\bx_i \stackrel{i.i.d}{\sim} \cH_d$ and $y_i = \prod_{i \in S} x_i + \xi_i$, for $\xi_i \sim \cN(0,\gamma^2)$. Let $\sigma \in S_d$ be a permutation that sends $\{1,\ldots,\floor{d/2}\}$ to $S$. Then by (a) permutation equivariance of $\ASGD_d$ from Proposition~\ref{prop:equi}; (b) letting $\tilde{\bq} = \sigma(\bq)$ and using that $\chi_S(\bq) = \prod_{i=1}^{\floor{d/2}} \tilde{\bq}_i$; (c) letting $(\tilde{\bx}_i, \tilde{y}_i)_{i \in [n_d]}$ be such that $\tilde{\bx}_i = \sigma(\bx_i)$ and $\tilde{y}_i = y_i$; and finally (d) using that $\tilde{\bq} \sim \cH_d$, and $(\tilde{\bx}_i,\tilde{y}_i)_{i \in [n_d]} \stackrel{i.i.d.}{\sim} \cD(f_*^{(d)}, \mu^{(d)}, \gamma)$ and a Markov bound on \eqref{eq:sgd-perm-hard-pre-markov}, ,
\begin{align*}
\PP\{\sgn([\ASGD_d((\bx_i,y_i)_{i \in [n_d]})](\bq)) = \chi_S(\bq)\}
&\stackrel{(a)}{=} \PP\{\sgn([\ASGD_d((\sigma(\bx_i),y_i)_{i \in [n_d]})](\sigma(\bq))) = \chi_S(\bq)\} \\
&\stackrel{(b)}{=} \PP\{\sgn([\ASGD_d((\sigma(\bx_i),y_i)_{i \in [n_d]})](\tilde{\bq})) = \prod_{i=1}^{\floor{d/2}} \tilde{q}_i\} \\
&\stackrel{(c)}{=} \PP\{\sgn([\ASGD_d((\tilde{\bx}_i,\tilde{y}_i)_{i \in [n_d]})](\tilde{\bq})) = \prod_{i=1}^{\floor{d/2}} \tilde{q}_i\} \\
&\stackrel{(d)}{\geq} 1 - 0.01 \geq 0.99.
\end{align*}
Therefore $\sgn(\ASGD_d)$ solves the $(d,n_d,\gamma)$-LPGN problem. By the $\gamma$-LPGN-hardness assumption this algorithm cannot run in $\poly(d)$ time. So if the $(\fNNd,\mu_{\btheta,d})$-SGD training can be run in $\poly(d)$ time and $\fNN(\cdot;\btheta_d)$ can be evaluated in $\poly(d)$ time, contradicting the $\gamma$-LPGN-hardness assumption.
\end{proof}

\subsection{Hardness from sign-flip equivariance, proof of Theorem~\ref{thm:sgd-sum-mod-8-hard}}

Our second result is not as obvious, and exploits the sign-flip equivariance of training on FC architectures with sign-flip-symmetric initialization (such as Gaussian initialization). Let $f_{\mathrm{mod8},d} : \cH_d \to \{0,\ldots,7\}$ denote the function given by $f_{\mathrm{mod8},d}(\bx) \equiv \sum_i x_i \pmod{8}$. For brevity, we drop the subscript with $d$ and write $\fmodeight$, since $d$ is clear from context.

Restated, our main result is:
\begin{theorem}[Theorem~\ref{thm:sgd-sum-mod-8-hard} restated]\label{thm:sgd-sum-mod-8-hard-restated}
Let $\{\fNNd, \mu_{\btheta,d}\}_{d \in \N}$ be a family of networks and initializations satisfying Assumption~\ref{ass:noskip} (fully-connected) with i.i.d. symmetric initialization. Let $\gamma > 0$, and let $\{n_d\}$ be sample sizes such that $(\fNNd,\mu_{\btheta,d})$-SGD training on $n_d$ samples from $\cD(\fmodeight,\cH_d,\gamma)$ rounded to $\poly(d)$ bits yields parameters $\btheta_d$ with 
\begin{align}\label{eq:sgd-sum-mod-8-hard-pre-markov}
\E_{\btheta_{d}}[\|\fmodeight - \fNN(\cdot;\btheta_d)\|^2] \leq 0.0001.
\end{align}
Then, under $(\gamma/2)$-LPGN hardness, it is not possible to run $(\fNNd,\mu_{\btheta,d})$-SGD on $n_d$ samples and evaluate $\fNN(\cdot;\btheta_d)$ in $\poly(d)$ time.
\end{theorem}

The proof of Theorem~\ref{thm:sgd-sum-mod-8-hard} can be also proved via a reduction to LPGN. However, this reduction is much less straightforward.

\subsubsection{The secretly-flipped sum-mod-8 (SFSM8) problem}

We prove the hardness of learning sum-mod-8 with SGD by using only the sign-flip equivariance, and no other properties of the SGD algorithm. The idea is that any sign-flip-equivariant algorithm that can learn $\fmodeight$, must also be capable of solving a much more difficult problem. First, define the problem of outputting sum-mod-8. 
\begin{problem}
The $(d,n,\gamma)$-SM8 (sum-mod-8) problem is parametrized by $\gamma > 0$ and integers $d,n > 0$. It is as follows:
\begin{itemize}
    \item Input: query vector $\bq \sim \cH_d$, samples $(\bx_i,y_i)_{i \in [n]} \stackrel{i.i.d.}{\sim} \cD(\fmodeight,\cH_d,\gamma)$.
\item Task: return $\fmodeight(\bq) \in \{0,\ldots,7\}$.
\end{itemize}
\end{problem}

Notice that SM8 is a trivial problem: an algorithm that was not sign-flip equivariant could ignore the samples $(\bx_i,y_i)_{i \in [n]}$, and simply return $\fmodeight(\bq)$. However, sign-flip equivariance makes solving SM8 much more difficult: we prove that any sign-flip equivariant algorithm that can solve SM8, can also solve the problem SFSM8, defined below.

\begin{problem}
The $(d,n,\gamma)$-secretly-flipped sum-mod-8 (SFSM8) problem is parametrized by $\gamma > 0$, and integers $d,n > 0$. It is as follows:
\begin{itemize}
    \item Unknown: sign-flip vector $\bs \in \cH_d$.
    \item Input: query vector $\bq \sim \cH_d$, modified samples $(\bx_i \odot \bs,y_i)_{i \in [n]}$, where $(\bx_i,y_i)_{i \in [n]} \stackrel{i.i.d.}{\sim} \cD(\fmodeight,\cH_d,\gamma)$.
    \item Task: return $\fmodeight(\bq \odot \bs) \in \{0,\ldots,7\}$.
\end{itemize}
\end{problem}

\begin{lemma}\label{lem:sgd-sign-helper}
Let $\cA$ be a $\Gsign$-equivariant algorithm in the sense of Definition~\ref{def:equi-sgd}. Then $\cA$'s error probability on SM8 equals $\cA$'s error probability on SFSM8.
\end{lemma}
\begin{proof}
By sign-flip equivariance, for any $\bs,\bq$, and any samples $(\bx_i,y_i)_{i \in [n]}$,
$$[\cA((\bx_i \odot \bs,y_i)_{i \in [n]})](\bq\odot \bs) \stackrel{d}{=} [\cA((\bx_i,y_i)_{i\in [n]})](\bq).$$
So drawing $(\bx_i,y_i)_{i \in [n]}  \stackrel{i.i.d.}{\sim} \cD(\fmodeight,\cH_d,\gamma)$, $$\PP[[\cA((\bx_i,y_i)_{i \in [n]})](\bq) = \fmodeight(\bq)] = \PP[[\cA((\bx_i \odot \bs,y_i)_{i \in [n]})](\bq \odot \bs) = \fmodeight(\bq \odot \bs)].$$
\end{proof}

\subsubsection{Reduction from LPGN to SFSM8}
Now we show that there is no algorithm that solves SFSM8 in polynomial time and samples, under the cryptographic assumption that learning parities with Gaussian noise (LPGN) is hard.

\begin{lemma}[Reduction from LPGN to SFSM8]\label{lem:lpgn-to-sfsm8}
Given access to an oracle $\cA$ for $(d,n,\gamma_{\cA})$-SFSM8 that is correct with probability $> 9/10$, one can construct an algorithm $\cB$ for $(d,n,\gamma_{\cA}/2)$-LPGN that is correct with probability $> 9/10$ and runs in one call to $\cA$ and $\poly(n,d)$ additional time. Furthermore, $\gamma_{\cB} = \gamma_{\cA} / 2$.
\end{lemma}
\begin{figure}[t!]
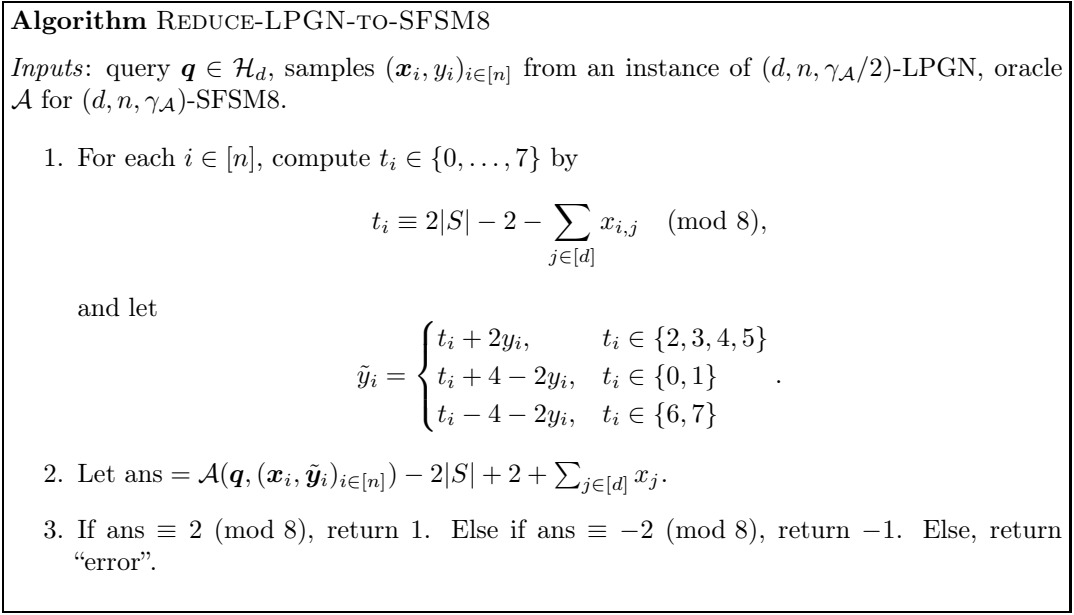

\centering
\begin{algbox}
\textbf{Algorithm} \textsc{Reduce-LPGN-to-SFSM8}

\vspace{2mm}

\textit{Inputs}: query $\bq \in \cH_d$, samples $(\bx_i,y_i)_{i \in [n]}$ from an instance of $(d,n,\gamma_{\cA}/2)$-LPGN, oracle $\cA$ for $(d,n,\gamma_{\cA})$-SFSM8.

\begin{enumerate}
\item For each $i \in [n]$, compute $t_i \in \{0,\ldots,7\}$ by $$t_i \equiv 2|S| - 2 - \sum_{j \in [d]} x_{i,j} \pmod{8},$$ and let $$\tilde{y}_i = \begin{cases} t_i + 2y_i, & t_i \in \{2,3,4,5\} \\ t_i + 4 - 2y_i, & t_i \in \{0,1\} \\ t_i - 4 - 2y_i, & t_i \in \{6,7\} \end{cases}.$$

\item Let $\mbox{ans} = \cA(\bq, (\bx_i,\tilde{\by}_i)_{i \in [n]}) - 2|S| + 2 + \sum_{j \in [d]} x_j$.

\item If $\mbox{ans} \equiv 2 \pmod{8}$, return $1$. Else if $\mbox{ans} \equiv -2 \pmod{8}$, return $-1$. Else, return ``error''.
\end{enumerate}
\vspace{1mm}
\end{algbox}
\caption{Reduction from LPGN to SFSM8 (Lemma~\ref{lem:lpgn-to-sfsm8}). Here, $\gamma_{\cB} = \gamma_{\cA}/2$. Note that $|S| = \floor{d/2}$ is known, so the reduction is efficient.}
\label{fig:lpgn-to-sfsm8}
\end{figure}
\begin{proof}
The pseudocode for the reduction is given in Figure~\ref{fig:lpgn-to-sfsm8}. We prove correctness.
Let $S \subseteq [d]$ be the unknown subset for the LPGN problem. Define $\bs \in \cH_d$ where $s_j = 1$ for all $j \in S$, and $s_j = -1$ for all $j \not\in S$. Then, for any $\bx \in \cH_d$,
\begin{align*}
\sum_{j \in [d]} x_j s_j &= \sum_{j \in [d]} x_j(1+s_j) - \sum_{j \in [d]} x_j \\
&= 2|\{j \in S : x_j = 1\}| - 2|\{j \in S : x_j = -1\}| - \sum_{j \in [d]} x_j \\
&= 2(|S| - 2|\{j \in S : x_j = -1\}|) - \sum_{j \in [d]} x_j \\
&\equiv 2 \prod_{j \in S} x_j + 2|S| - 2 - \sum_{j \in [d]} x_j \pmod{8}
\end{align*}
where in the last line we use $|\{j \in S : x_j = -1\}| \equiv (\prod_{j \in S} x_j - 1)/2 \pmod{2}$.

This guarantees that after Step 1 of the algorithm, the samples $(\bx_i,\tilde{y}_i)_{i \in n}$ are distributed as if they were drawn from an instance of $(d,n,\gamma_{\cA})$-SFSM8 with secret sign-flip vector $\bs$. I.e., $\bx_i \stackrel{i.i.d.}{\sim} \cH_d$, and $\tilde{y}_i = \fmodeight(\bx_i) + \xi_i$ for $\xi_i  \stackrel{i.i.d}{\sim} \cN(0,\gamma_{\cA}^2)$.

By the correctness guarantee of $\cA$, we know that with probability at least 9/10, the output $\cA(\bq,(\bx_i,\tilde{y}_i)_{i \in [n]}) \in \{0,\ldots,7\}$ satisfies $\cA(\bq,(\bx_i,\tilde{y}_i)_{i \in [n]}) \equiv \sum_i q_i s_i \pmod{8}$. So by the above calculations, $\mbox{ans} \equiv 2\prod_{j \in S} q_S$, proving correctness.
\end{proof}

\subsubsection{Proof of Theorem~\ref{thm:sgd-sum-mod-8-hard}}

\begin{proof}[Proof of Theorem~\ref{thm:sgd-sum-mod-8-hard}]
Suppose by contradiction that $(\fNNd,\mu_{\btheta,d})$-SGD can be run in $\poly(d)$ time and the learned function can be evaluated in $\poly(d)$ time. Let $\cA_d = \mathrm{round}(\fNNd(\cdot;\btheta_d))$ be the algorithm that runs $(\fNNd, \mu_{\btheta,d})$-SGD on $n_d$ samples $(\bx_i,y_i)_{i \in [n_d]}$ and returns the learned function, rounded to the nearest integer. By a Markov bound on \eqref{eq:sgd-sum-mod-8-hard-pre-markov}, $\PP\{[\cA_d((\bx_i,y_i)_{i \in [n_d]})](\bq) = \fmodeight(\bq)] \geq 1 - 0.004 \geq 0.99$. Furthermore, $\cA_d$ is $\Gsign$-equivariant by Proposition~\ref{prop:equi}. So by Lemma~\ref{lem:sgd-sign-helper}, $\cA_d$ gives a $\poly(d)$-size circuit for the $(d,n_d,\gamma)$-SFSM8 problem. By the LPGN-to-SFSM8 reduction in Lemma~\ref{lem:lpgn-to-sfsm8}, we see that this contradicts the $(\gamma/2)$-LPGN-hardness assumption.
\end{proof}

\subsection{On the cryptographic assumption that LPGN is hard}\label{app:lpgn}

In our SGD hardness results, we assume hardness of the LPGN (Learning Parities with Gaussian Noise) problem from Definition~\ref{def:lpgn}. The standard hardness assumption in the literature is hardness of LPN (Learning Parities with Noise). The differences are that in LPN: (1) the noise is classification noise instead of Gaussian noise; (2) there is no promise that the unknown subset has size $|S| = \floor{d/2}$. In this appendix, we show that our LPGN assumption can be derived from the LPN assumption, with slightly different parameters.

\begin{definition}\label{def:lpn}
The learning parities with noise, $(d,n,\rho)$-LPN, problem is parametrized by $d,n \in \ZZ_{> 0}$ and $\rho \in (0,1]$. An instance $(S, \bq, (\bx_i,y_i)_{i \in [n]})$ consists of
(i) an unknown subset $S \subseteq [d]$, and (ii) a known query vector $\bq \sim \cH_d$, and i.i.d. samples $(\bx_i,y_i)_{i \in [n]}$ such that $\bx_i \sim \cH_d$ and $y_i = \chi_S(\bx_i)\zeta_i$, where $$\zeta_i \sim \begin{cases} 1, & \mbox{ w.p. } (1+\rho)/2 \\ -1, & \mbox{ w.p. } (1-\rho)/2 \end{cases}.$$ The task is to return $\chi_S(\bq) \in \{+1,-1\}$.
\end{definition}

In order to reduce from LPN to LPGN, let us define promise-LPN:
\begin{definition}\label{def:promise-lpn}
The promise-LPN problem is the LPN problem with the promise that $|S| = \floor{d/2}$.
\end{definition}

We prove the following theorem, where $\rho(d)$-LPN is the LPN problem where can take any number of samples $n$, and the value of $\rho$ depends on $d$.
\begin{theorem}
Suppose that for all constants $C > 0$, and any $\rho(d) \leq 1 - \exp(-C\sqrt{\log(d)})$, there are no $\poly(d)$-time algorithms for $\rho(d)$-LPN. Then, for any constant $\gamma > 0$, there are no $\poly(d)$-time algorithms for $\gamma$-LPGN.
\end{theorem}
\begin{proof}
This follows by combining Lemmas~\ref{lem:lpn-to-promise-lpn} and \ref{lem:promise-lpn-to-lpgn}, as outlined in the diagram below:
\begin{align*}
\mbox{LPN} \stackrel{\mathrm{Lemma}~\ref{lem:lpn-to-promise-lpn}}{\longrightarrow} \mbox{promise-LPN} \stackrel{\mathrm{Lemma}~\ref{lem:promise-lpn-to-lpgn}}{\longrightarrow} \mbox{LPGN}\,.
\end{align*}
\end{proof}

\subsubsection{Reduction from LPN to promise-LPN}\label{app:LPN-to-promise-LPN}

\begin{figure}[t!]
\centering
\begin{algbox}
\textbf{Algorithm} \textsc{LPN-to-promise-LPN}

\vspace{2mm}

\textit{Inputs}: query $\bq \in \cH_d$, samples $(\bx_i,y_i)_{i \in [n_{\cB}]}$, from an instance of $(d,\rho,n_{\cB})$-LPN, oracle $\cA$ for $(2d,\rho,n_{\cA})$-promise-LPN.

\begin{enumerate}
\item Let $T = 10000 \log(d) / \rho^2$. Relabel the samples, splitting them into groups as
$$(\bx_i,y_i)_{i \in [n_{\cB}]} =  \left(\bigsqcup_{r \in \{0,\ldots,d\}, t \in [T]} \{(\bx^{(r,t)}_i,y^{(r,t)}_i)\}_{i \in [n_{\cA}+1]}\right) \sqcup (\bx^{(\mathrm{\ast})}_i,y^{(\mathrm{\ast})}_i)_{i \in [n_{\cA}]}$$

\item For each $r \in \{0,\ldots,d\}$ and $t \in [T]$, let  $$\mathrm{ans}_{r,t} = \cA([\bx_{n_{\cA}+1}^{(r,t)}, \bz_{n_{\cA}+1}^{(r,t)}], ([\bx_i^{(r,t)}, \bz_i^{(r,t)}],y_i^{(r,t)} \cdot \prod_{j=1}^{r} z_{ij}^{(r,t)})_{i \in [n_{\cA}]}) \cdot \prod_{j=1}^r z_{n_{\cA}+1,j}^{(r,t)},$$
where $\bz_i^{(r,t)} \sim \cH_d$ is random padding and $[\bx_i^{(r,t)}, \bz_i^{(r,t)}] \in \cH_{2d}$ is the concatenation.
\item For each $r \in \{0,\ldots,d\}$, let $\hat{p}_r = |\{t \in [T] : \mathrm{ans}_{r,t} = y_{n_{\cA+1}}^{(r,t)}\}| / T$
\item Let $(\bz_{i}^{(\ast)})_{i \in [n_{\cA}]} \sim \cH_d$. Let $\hat{r} = \arg\max_r \hat{p}_r$. Return $$\cA([\bq, \bz_{n_{\cA}+1}^{(\ast)}], ([\bx_i^{(\ast)}, \bz_i^{(\ast)}],y_i^{(\ast)} \cdot \prod_{j=1}^{\hat{r}} z_{ij}^{(\ast)})_{i \in [n_{\cA}]}) \cdot \prod_{j=1}^{\hat{r}} z_{n_{\cA}+1,j}^{(\ast)}.$$
\end{enumerate}
\vspace{1mm}
\end{algbox}
\caption{Reduction from LPN to promise-LPN (Lemma~\ref{lem:lpn-to-promise-lpn}). Here $n_{\cB} = dT(1 + n_{\cA}) + n_{\cA}$.}
\label{fig:lpn-to-promise-lpn}
\end{figure}

\begin{lemma}[LPN to promise-LPN]\label{lem:lpn-to-promise-lpn} Given access to an oracle $\cA$ for $(2d,n_{\cA},\rho)$-promise-LPGN that is correct with probability $> 19/20$, one can construct an algorithm $\cB$ for $(d,n_{\cB},\rho)$-LPGN that is correct with probability $> 9/10$ and runs in $O(d\log(d) / \rho^2)$ calls to $\cA$ and $O(n_{\cA} d \log(d) / \rho^2)$ additional time. Furthermore, $n_{\cB} \leq d\log(d) n_{\cA} / \rho^2$.
\end{lemma}
\begin{proof}
The pseudocode for the reduction is given in Figure~\ref{fig:lpn-to-promise-lpn}. We prove correctness. Let $S \subseteq [d]$ be the unknown subset for the LPN problem. For any $r \in \{0,\ldots,d\}$, define the success probability of running $\cA$ with $r$ as $$s_r := \PP[\mathrm{ans}_{r,1} = \chi_S(\bx_{n_{\cA}+1}^{(r,1)})].$$ Notice that if we take $r^* = d-|S|$, then the samples that we feed into $\cA$ are those of the LPN problem with unknown subset $S' = S \cup \{d+1,\ldots,d+r\}$, which has size $|S'| = d = \floor{2d / 2}$. Therefore, the guarantee for $\cA$ implies that if we take $r^* = d-|S|$, then $$s_{r^*} > 19/20.$$ However, we are not given $r^*$ in the input of the LPGN problem, which is the main difficulty. Instead, in the first part of the algorithm we estimate the success probability for each $r \in \{0,\ldots,d\}$ using fresh samples, and at the end run $\cA$ with the $r^*$ that we estimate gives the best success probability.

To prove correctness, let $E$ be the event that for all $r \in \{0,\ldots,d\}$ we have 
$$|\hat{p}_r - \PP[\mathrm{ans}_{r,1} = y_{n_{\cA}+1}^{(r,1)}]| \leq 1/(50\rho).$$
By Hoeffding's inequality, $E$ holds with probability at least $\PP[E] \geq 1 - 2\exp(-8\log(d)) \geq 1/100$. Now notice that $y_{n_{\cA}+1}^{(r,1)}$ is just $\chi_S(\bx_{n_{\cA}+1}^{(r,1)})$ with classification noise so under the event $E$
\begin{align*}
|(\hat{p}_r + \rho/2 - 1/(2\rho)) / \rho - s_r| \leq 1/(50).
\end{align*}
Therefore, since $s_{r^*} \geq 19/20$, under event $E$ we must have $s_{\hat{r}} \geq 19/20 - 1/25 = 9/10 + 1/100$. Finally, since the probability of success of the algorithm is $s_{\hat{r}}$, and $\PP[E] \geq 1/100$, we have that the algorithm's success probability is at least $9/10$.
\end{proof}

\subsubsection{Reduction from promise-LPN to LPGN}\label{app:promise-LPN-to-LPGN}

\begin{figure}[t!]
\centering
\begin{algbox}
\textbf{Algorithm} \textsc{promise-LPN-to-LPGN}

\vspace{2mm}

\textit{Inputs}: query $\bq \in \cH_d$, samples $(\bx_i,y_i)_{i \in [n]}$, from an instance of $(d,\rho,n)$-promise-LPN, oracle $\cA$ for $(d,\gamma,n)$-LPGN.

\begin{enumerate}
\item For each $i \in [n]$, let $\tilde{y}_i = \textsc{RK}(y_i)$, where $\textsc{RK}$ is the rejection kernel $$\textsc{RK} = \textsc{RK}((1+\rho)/2 \to \cN(1,\gamma^2), (1-\rho)/2 \to \cN(-1,\gamma^2))$$ from Lemma~5.1 of \cite{brennan2018reducibility}.

\item Return $\cA(\bq, (\bx_i, \tilde{y}_i)_{i \in n})$
\end{enumerate}

\vspace{1mm}
\end{algbox}
\caption{Reduction from promise-LPN to LPGN (Lemma~\ref{lem:promise-lpn-to-lpgn}).}
\label{fig:promise-lpn-to-lpgn}
\end{figure}

Here, we show that the LPGN hardness assumption follows from the standard LPN hardness assumption with classification noise. The technique is rejection kernels.

\begin{lemma}\label{lem:promise-lpn-to-lpgn}
Given access to an oracle $\cA$ for $(2d,n,\gamma)$-LPGN that is correct with probability $> 49/50$, one can construct an algorithm $\cB$ for $(d,n,\rho)$-promise-LPN that is correct with probability $> 19/20$ and runs in one call to $\cA$ and $\poly(n,d)$ additional time. Furthermore, we can take $\rho \leq 1 - \exp(-C \sqrt{\log(n)} \max(\gamma^2, 1 / \gamma^2))$, for some universal constant $C > 0$.
\end{lemma}

\begin{proof}
The pseudocode of the reduction is given in Figure~\ref{fig:promise-lpn-to-lpgn}. We use the rejection kernel technique developed in \cite{brennan2018reducibility} to convert the classification noise of LPN to the additive Gaussian noise of LPGN. Let $\rho = 1 - 2\delta$, where $\delta < 0.1$. Notice that if we define $p = (1+\rho)/2 = 1 - \delta$ and $q = (1-\rho)/2 = \delta$, Lemma~5.1 of \cite{brennan2018reducibility} provides a randomized map $\textsc{RK}$ that runs in time and maps a $\Rad(p)$ random variable to $\cN(1,\gamma^2)$ and a $\Rad(q)$ random variable to $\cN(-1,\gamma^2)$. For some parameter $N > 0$, this map runs in $O(N)$ time and has the guarantee that $$d_{\mathrm{TV}}(\textsc{RK}(\Rad(p)), \cN(1,\gamma^2)) \leq \Delta\quad \mbox{and} \quad d_{\mathrm{TV}}(\textsc{RK}(\Rad(q)),\cN(-1,\gamma^2)) \leq \Delta,$$
where $$\Delta = \PP_{X \sim \cN(1,\gamma^2)}[X \not\in \cS] + (\PP_{X \sim \cN(-1,\gamma^2)}[X \not\in \cS] + \frac{\delta}{1-\delta})^N,$$
and $$\cS = \{x : \frac{\delta}{1-\delta} \leq \frac{\exp(-(x-1)^2 / (2\gamma^2)}{\exp(-(x+1)^2/(2\gamma^2)} \leq \frac{1-\delta}{\delta}\} = \{x : \frac{\delta}{1-\delta} \leq \exp(\frac{2 x }{ \gamma^2}) \leq \frac{1-\delta}{\delta}\}.$$
By standard Gaussian tail bounds, \begin{align*}\PP_{X \sim \cN(1,\gamma^2)}[X \not\in \cS] &= \PP_{X \sim \cN(-1,\gamma^2)}[X \not\in \cS] \leq \exp(-(\gamma^2\log((1-\delta)/\delta) / 2 - 1)^2 / 2\gamma^2).
\end{align*}
So for any $\eps > 0$, as long as $\delta < \min(\exp(-10 \sqrt{\log(1/\eps)} / \gamma^2), \exp(-10 \sqrt{\log(1/\eps)} \gamma^2))$, we have
\begin{align*}
\PP_{X \sim \cN(1,\gamma^2)}[X \not\in \cS] \leq \eps.
\end{align*}
So letting $\eps \leq 1/(2000n^2)$, and $N \geq \log(2000n^2)$, we have $\Delta \leq n / 100$.

Thus, $(\bq, (\bx_i,\tilde{y}_i)_{i \in [n]})$ is $\leq 1/100$ total-variation distance from being drawn from LPGN with unknown set $S$. This means that the algorithm has success probability $49/50 - 1/100 \geq 19/20$.
\end{proof}

\section{On the equivariance of SGD and GD}\label{app:equi}

Recall from Definition~\ref{def:equi-intro} that an algorithm $\cA$ that takes in a distribution $\cD \in \cP(\cX \times \R)$ and outputs a function $\cA(\cD) : \cX \to \R$ is said to be $G$-equivariant if
\begin{align*}
\cA(\cD) \stackrel{d}{=} \cA(g(\cD)) \circ g,
\end{align*}
for any $g \in G$. Here we view the group element as a function $g : \cX \to \cX$, since it acts on the space of inputs $\cX$. We also define $g(\cD)$ to be the distribution of $(g(\bx),y)$, where $(\bx,y) \sim \cD$. In Definitions~\ref{def:equi-gd} and~\ref{def:equi-sgd}, we define the equivariance of GD and SGD, respectively. In Proposition~\ref{prop:equi}, we claim that GD and SGD are $G_{perm}$-equivariant in the case of training FC networks with i.i.d. initialization. Furthermore, these algorithms are $\Gsignperm$-, and $\Grot$-equivariant when the initialization is symmetric and Gaussian, respectively. The equivariances of SGD in Proposition~\ref{prop:equi} are proved by \cite{ng2004feature,li2021convolutional}. They also show extension to other algorithms beyond SGD:
\begin{remark}\label{rem:beyond-sgd}
\cite{li2021convolutional} note that other popular optimizers on FC networks with i.i.d. symmetric initialization, such as AdaGrad and Adam, also yield $\Gsignperm$-equivariance, as well as SGD with a mini-batch, or gradient descent on the loss of the empirical distribution $\ell_{\hat{\cD}}$, where $\hat{\cD} = \frac{1}{n} \sum_{i=1}^n \delta_{(y_i, \bx_i)}$, for $(\bx_i,y_i)_{i \in [n]} \stackrel{i.i.d.}{\sim} \cD$. Therefore our SGD hardness result in Theorem~\ref{thm:sgd-sum-mod-8-hard} applies to training with the above algorithms.
\end{remark}

We limit ourselves in this appendix to prove the equivariances claimed by Proposition~\ref{prop:equi} for GD. We provide the proofs only for completeness, since they are quite similar to the proofs of equivariance of SGD.

In the remainder of this section, suppose that we have an architecture $\fNN(\cdot;\btheta)$ and an initialization $\mu_{\btheta}$ that satisfy Assumption~\ref{ass:noskip}, so that we can write the parameters $\btheta = (\bW, \bpsi)$. Furthermore, let $\cD \subseteq \cP(\R^d \times \R)$ be a data distribution.
\begin{definition}
For any invertible linear transformation $\bM \in GL(d,\R)$, let $\bM$ act on $\btheta$ as $\bM \square \btheta = (\bW \bM, \bpsi)$.
Let $\bM \square \cD$ be the distribution of $(\bM \bx, y)$ for $(\bx, y) \sim \cD$.
\end{definition}
\begin{lemma}\label{lem:equi-gd-helper}
Suppose that for some orthogonal transformation $\bM \in O(d)$, if we draw $\btheta^0 \sim \mu_{\btheta}$, then $\bM \square \btheta^0 \stackrel{d}{=} \btheta^0$.
Let $\btheta^k$ be the weights from $(\fNN, \mu_{\btheta})$-GD on distribution $\cD$, for any number of steps $k$, with any learning rate $\eta > 0$, and with any noise level $\tau > 0$. Similarly, let $\tilde{\btheta}^k$ be the weights from running $(\fNN,\mu_{\btheta})$-GD on distribution $\tilde{\cD} = \bM \square \cD$.

Then, for any $\bx \in \R^d$,
\begin{align*}
\fNN(\bx; \btheta^k) \stackrel{d}{=} \fNN(\bM \bx; \tilde{\btheta}^k)
\end{align*}
\end{lemma}
\begin{proof}
We claim that we can couple $\btheta^k$ and $\tilde{\btheta}^k$ such that $\btheta^k = \bM \square \tilde{\btheta}^k$ almost surely. Therefore, for any $\bx$:
\begin{align}\label{eq:equi-gd-linear-helper}
\fNN(\bx;\btheta^k) = \fNN(\bx; \bM \square \tilde{\btheta}^k) = \gNN(\tilde{\bW}^k \bM \bx; \tilde{\bpsi}^k) = \fNN(\bM \bx; \tilde{\btheta}^k),
\end{align}
which proves the lemma. It remains to show our coupling inductively on $k$. The base case $k = 0$ is assumed. For the inductive step, assume that the coupling is true for $k \geq 0$, and prove it for $k+1$. By (a) using \eqref{eq:equi-gd-linear-helper},
\begin{align*}\bg_{\tilde{\cD}}(\tilde{\btheta}^{k}) &= \E_{(\bx,y) \sim \cD}[(\fNN(\bM \bx; \tilde{\btheta}^{k}) - y) \Pi_{B(0,R)}\nabla_{\btheta} \fNN(\bM \bx; \tilde{\btheta}^{k})] \\
&\stackrel{(a)}{=} \E_{(\bx,y) \sim \cD}[(\fNN(\bx;\btheta^{k}) - y) \Pi_{B(0,R)}\nabla_{\btheta} \fNN(\bM \bx; \tilde{\btheta}^{k})]
\end{align*}
Furthermore, by (a) the chain rule for differentiation combined with \eqref{eq:equi-gd-linear-helper}, (b) using that $\bM^{\top} \in O(d)$ preserves the $\|\cdot\|_2$ norm,  
\begin{align*}
\Pi_{B(0,R)}\nabla_{\btheta} \fNN(\bM \bx; \tilde{\btheta}^{k}) &= \Pi_{B(0,R)}[\nabla_{\bW} \gNN(\tilde{\bW}^k \bM \bx, \tilde{\bpsi}^k), \nabla_{\bpsi} \gNN(\tilde{\bW}^k \bM \bx, \tilde{\bpsi}^k)] \\
&\stackrel{(a)}{=} \Pi_{B(0,R)}[\nabla_{\bW} \gNN(\bW^k \bx, \bpsi^k) \bM^{\top}, \nabla_{\bpsi} \gNN(\bW^k \bx, \bpsi^k)] \\
&= \bM^{\top} \square (\Pi_{B(0,R)} [\nabla_{\bW} \gNN(\bW^k \bx, \bpsi^k), \nabla_{\bpsi} \gNN(\bW^k \bx, \bpsi^k)])
\end{align*}
By linearity of the $\square$ operation, it follows that
\begin{align*}
\bg_{\tilde{\cD}}(\tilde{\btheta}^k) = \bM^{\top} \square \bg_{\cD}(\btheta^k).
\end{align*}
So since $\bM \in O(d)$ we have $\bM^{-1} = \bM^{\top}$, implying $$\bg_{\cD}(\btheta^k) = \bM \square \bg_{\tilde{\cD}}(\tilde{\btheta}^k).$$ Also couple the added noises $\xi^k$ and $\tilde{\xi}^k$ so that $\xi^k = \bM \square \tilde{\xi}^k$, which can be done since $\bM \in O(d)$ and Gaussians are orthogonal-invariant. By linearity of the $\square$ operation, and the inductive hypothesis, it holds that $\btheta^{k+1} = \bM \square \tilde{\btheta}^{k+1}$ almost surely. The proof of the claim follows by induction.
\end{proof}

The above lemma immediately implies the GD equivariances claimed in Proposition~\ref{prop:equi}.
\begin{proof}[Proof of GD equivariance in Proposition~\ref{prop:equi}]
Notice that $G_{perm}$, $\Gsignperm$, and $\Grot$ can be identified with subgroups of $O(d)$. For i.i.d. initialization $\mu_{\bW} = \mu_w^{\otimes (m \times d)}$, for any permutation matrix $\bM$ we have $\bM \square \btheta^0 \stackrel{d}{=} \btheta^0$. Similarly, if $\mu_{w}$ is symmetric then for any signed permutation matrix $\bM$ we have $\bM \square \btheta^0 \stackrel{d}{=} \btheta^0$. Finally, if $\mu_w$ is Gaussian then for any rotation matrix $\bM \in SO(d)$ we have $\bM \square \btheta^0 \stackrel{d}{=} \btheta^0$. Lemma~\ref{lem:equi-gd-helper} implies the GD equivariances claimed in Proposition~\ref{prop:equi}.
\end{proof}

\end{document}

%% file: importsanddeclares.tex
\usepackage{amsmath}
\usepackage{amsthm}
\usepackage{amssymb}
\usepackage{color}
\usepackage{mathtools}
\usepackage{makecell}
\usepackage{bm}
\usepackage{todonotes}
\usepackage{diagbox}
\usepackage{tablefootnote}
\usepackage[colorlinks=true,linkcolor=blue]{hyperref}
\usepackage[ruled,vlined,linesnumbered]{algorithm2e}
\usepackage{layouts}
\usepackage{algorithmic}
\usepackage{amsmath}
\usepackage{amsthm}
\usepackage{amssymb}
\usepackage{color}
\usepackage{mathtools}
\usepackage{makecell}
\usepackage{bm}
\usepackage{todonotes}
\usepackage{diagbox}
\usepackage{tablefootnote}
\usepackage{fancybox}
\usepackage[bottom]{footmisc}
\usepackage{appendix}

\DeclarePairedDelimiter{\floor}{\lfloor}{\rfloor}

\DeclareMathOperator{\E}{\mathbb{E}}

\DeclareMathOperator{\sgn}{sgn}
\DeclareMathOperator{\poly}{poly}

\DeclareMathOperator{\Rad}{Rad}

\let\baraccent=\= %
\renewcommand{\=}[1]{\stackrel{#1}{=}} %

\let\S=\Sec
\providecommand{\S}{\mathbb{S}}
\providecommand{\R}{\mathbb{R}}
\providecommand{\NN}{\mathbb{N}}
\providecommand{\CC}{\mathbb{C}}

\providecommand{\ZZ}{\mathbb{Z}}

\providecommand{\cA}{\mathcal{A}}
\providecommand{\cB}{\mathcal{B}}

\providecommand{\cD}{\mathcal{D}}

\providecommand{\cF}{\mathcal{F}}

\providecommand{\cL}{\mathcal{L}}

\providecommand{\cN}{\mathcal{N}}
\providecommand{\cP}{\mathcal{P}}

\providecommand{\cS}{\mathcal{S}}
\providecommand{\cV}{\mathcal{V}}
\providecommand{\cW}{\mathcal{W}}
\providecommand{\cX}{\mathcal{X}}
\providecommand{\cY}{\mathcal{Y}}

\providecommand{\PP}{\mathbb{P}}
\providecommand{\EE}{\mathbb{E}}

\providecommand{\eps}{\epsilon}

\providecommand{\N}{\mathbb{N}}

\mathchardef\mhyphen="2D %

\providecommand{\ba}{\bm{a}}

\providecommand{\bw}{\bm{w}}
\providecommand{\bx}{\bm{x}}
\providecommand{\by}{\bm{y}}

\providecommand{\sm}{\setminus}

\providecommand{\cH}{\mathcal{H}}

\newcommand{\interior}[1]{%
  {\kern0pt#1}^{\mathrm{o}}%
}

\usepackage[capitalise]{cleveref}
\newtheorem{theorem}{Theorem}[section]
\newtheorem{lemma}[theorem]{Lemma}

\newtheorem{claim}[theorem]{Claim}
\newtheorem{proposition}[theorem]{Proposition}
\newtheorem{corollary}[theorem]{Corollary}
\newtheorem{problem}[theorem]{Problem}
\newtheorem{assumption}[theorem]{Assumption}

\newtheorem{definition}[theorem]{Definition}

\newtheorem{remark}[theorem]{Remark}

\renewcommand{\Im}{\mathrm{Im}}

\newcommand{\Djunk}{\cD_{\mathrm{junk}}}

\newcommand{\TV}{\mathrm{TV}}
\newcommand{\KL}{\mathrm{KL}}

\def\bI{{\boldsymbol I}}

\def\bM{{\boldsymbol M}}

\def\bW{{\boldsymbol W}}

\def\ba{{\boldsymbol a}}

\def\bg{{\boldsymbol g}}

\def\bq{{\boldsymbol q}}

\def\bs{{\boldsymbol s}}

\def\bv{{\boldsymbol v}}
\def\bw{{\boldsymbol w}}
\def\bx{{\boldsymbol x}}
\def\by{{\boldsymbol y}}
\def\bz{{\boldsymbol z}}

\def\bpsi{{\boldsymbol \psi}}

\def\btheta{{\boldsymbol \theta}}

\def\bxi{{\boldsymbol \xi}}

\def\bzero{{\boldsymbol 0}}

\def\bone{{\boldsymbol 1}}

\newcommand{\<}{\langle}
\renewcommand{\>}{\rangle}

\newcommand{\fNN}{f_{\mathsf{NN}}}
\newcommand{\fNNd}{f_{\mathsf{NN},d}}

\newcommand{\gNN}{g_{\mathsf{NN}}}

\newcommand{\fmodeight}{f_{\mathrm{mod8}}}

\newenvironment{fminipage}%
  {\begin{Sbox}\begin{minipage}}%
  {\end{minipage}\end{Sbox}\fbox{\TheSbox}}

\newenvironment{algbox}[0]{
\noindent 
\begin{fminipage}{5.5in}
}{
\end{fminipage}
}

\def\Gsign{{G_{\mathrm{sign}}}}
\def\Gperm{{G_{\mathrm{perm}}}}
\def\Grot{{G_{\mathrm{rot}}}}
\def\Gsignperm{{G_{\mathrm{sign,perm}}}}

\newcommand{\AGD}{\mathcal{A}^{\mathrm{GD}}}
\newcommand{\ASGD}{\mathcal{A}^{\mathrm{SGD}}}

\newcommand{\Capitalltwo}{L^2}